\documentclass{article}

\usepackage{microtype}
\usepackage{enumerate}
\usepackage{amsmath}
\usepackage{mathrsfs}
\usepackage{amsfonts}
\usepackage{amsthm}
\usepackage{booktabs}
\usepackage{natbib}
\usepackage{stmaryrd}
\usepackage{enumitem}
\usepackage{graphicx}
\usepackage{bm}
\usepackage{wrapfig}

\usepackage{algorithm}
\usepackage{algorithmic}

\usepackage[textsize=tiny]{todonotes}

\usepackage{hyperref}

\usepackage[preprint]{neurips_2019}

\newcommand{\cX}{\mathcal{X}}

\newcommand{\cS}{\mathcal{S}}

\newcommand{\cA}{\mathcal{A}}

\newcommand{\cP}{\mathcal{P}}

\newcommand{\cT}{\mathcal{T}}

\newcommand{\cV}{\mathcal{V}}

\newcommand{\cJ}{\mathcal{J}}
\newcommand{\cC}{\mathcal{C}}

\newcommand{\bR}{\mathbb{R}}

\newcommand{\bP}{\mathbb{P}}

\newcommand{\indic}[1]{\mathbb{I}_{\left [ #1 \right ]}}
\DeclareMathOperator*{\expect}{{\huge \mathbb{E}}}

\newcommand{\norm}[1]{\big \| #1 \big \|}

\newtheorem{prop}{Proposition}
\newtheorem{lem}{Lemma}

\newtheorem{cor}{Corollary}
\newtheorem{thm}{Theorem}

\newcommand{\cbar}{\, | \,}

\DeclareMathOperator*{\argmin}{arg\,min}
\DeclareMathOperator*{\argmax}{arg\,max}

\newcommand{\eqnref}[1]{(\ref{eqn:#1})}

\title{A Geometric Perspective on Optimal Representations for Reinforcement Learning}

\author{
  Marc G. Bellemare$^1$, Will Dabney$^2$, Robert Dadashi$^1$, Adrien Ali Taiga$^{1,3}$,\\
  \textbf{Pablo Samuel Castro$^1$, Nicolas Le Roux$^1$, Dale Schuurmans$^{1,4}$, Tor Lattimore$^2$, Clare Lyle$^5$}
}

\makeatletter
\def\blfootnote{\gdef\@thefnmark{}\@footnotetext}
\makeatother

\begin{document}

\maketitle

\def \Vpi {{V^\pi}}
\def \Vmu {{V^\mu}}
\def \hVpi {\hat{V}^\pi}
\def \hVphipi {\hat{V}_\phi^\pi}
\def \hVphimu {\hat{V}_\phi^\mu}
\def \hVphimui{\hat{V}_\phi^{\mu_i}}
\def \Piphi {\Pi_{\phi}}
\def \Phispace {\bR^{n \times d}}
\def \thetapi {\theta_\pi}

\def \rlp {\textsc{rlp}}

\def \SetPhi {\mathscr{R}}
\def \Conv {\textsc{Conv}}

\def \SetPi {\cP}

\def \Ppi {{P^\pi}}
\def \bPd {\bP_d}
\def \bPv {\bP_v}
\def \cVd {\cV_d}
\def \cVv {\cV_v}

\def \bmu {\bm{\mu}}
\def \bbV {\bm{V}}
\def \dpi {d_{\pi}}
\def \tpi {\tilde \pi}
\def \Vtpi {V^{\tilde \pi}}

\def \dSet {\mathscr{P}}
\def \hull {\textsc{hull}}

\begin{abstract}
We propose a new perspective on representation learning in reinforcement learning based on geometric properties of the space of value functions. We leverage this perspective to provide formal evidence regarding the usefulness of value functions as auxiliary tasks.
Our formulation considers adapting the representation to minimize the (linear) approximation of the value function of all stationary policies for a given environment. We show that this optimization reduces to making accurate predictions regarding a special class of value functions which we call \emph{adversarial value functions} (AVFs). 
We demonstrate that using value functions as auxiliary tasks corresponds to an expected-error relaxation of our formulation, with AVFs a natural candidate, and identify a close relationship with proto-value functions (Mahadevan, 2005). We highlight characteristics of AVFs and their usefulness as auxiliary tasks in a series of experiments on the four-room domain.
\end{abstract}

\section{Introduction}

\blfootnote{$^1$Google Brain $^2$DeepMind $^3$Mila, Universit\'e de Montr\'eal $^4$University of Alberta $^5$University of Oxford}
A good representation of state is key to practical success in reinforcement learning.
While early applications used hand-engineered features \citep{samuel59some}, these have proven onerous to generate and difficult to scale.
As a result, methods in representation learning have flourished, ranging from basis adaptation \citep{menache05basis,keller06automatic}, gradient-based learning \citep{yu09basis}, proto-value functions \citep{mahadevan07proto}, feature generation schemes such as tile coding \citep{sutton96generalization} and the domain-independent features used in some Atari 2600 game-playing agents \citep{bellemare13arcade,liang16state}, and nonparametric methods \citep{ernst05treebased,farahmand16regularized,tosatto17boosted}.
Today, the method of choice is deep learning. Deep learning has made its mark by showing it can learn complex representations of relatively unprocessed inputs using gradient-based optimization \citep{tesauro95temporal,mnih15human,silver16mastering}.

Most current deep reinforcement learning methods augment their main objective with additional losses called \emph{auxiliary tasks}, typically with the aim of facilitating and regularizing the representation learning process. The UNREAL algorithm, for example, makes predictions about future pixel values \citep{jaderberg17reinforcement}; recent work approximates a one-step transition model to achieve a similar effect \citep{francoislavet18combined,gelada19deepmdp}. The good empirical performance of distributional reinforcement learning \citep{bellemare17distributional} has also been attributed to representation learning effects, with recent visualizations supporting this claim \citep{such19atari}. However, while there is now conclusive empirical evidence of the usefulness of auxiliary tasks, their design and justification remain on the whole ad-hoc. One of our main contributions is to provides a formal framework in which to reason about auxiliary tasks in reinforcement learning.

We begin by formulating an optimization problem whose solution is a form of optimal representation. Specifically, we seek a state representation from which we can best approximate the value function of any stationary policy for a given Markov Decision Process. Simultaneously, the largest approximation error in that class serves as a measure of the quality of the representation. 
While our approach may appear naive -- in real settings, most policies are uninteresting and hence may distract the representation learning process -- we show
that our representation learning problem can in fact be restricted to a special subset of value functions which we call \emph{adversarial value functions} (AVFs).
We then characterize these adversarial value functions and
show they correspond to deterministic policies that either minimize or maximize the expected return at each state, based on the solution of a network-flow optimization derived from an interest function $\delta$. 

A consequence of our work is to formalize why predicting value function-like objects is helpful in learning representations, as has been argued in the past \citep{sutton11horde,sutton16emphatic}. We show how using these predictions as auxiliary tasks can be interpreted as a relaxation of our optimization problem. From our analysis, we hypothesize that auxiliary tasks that resemble adversarial value functions should give rise to good representations in practice.
We complement our theoretical results with an empirical study in a simple grid world environment, focusing on the use of deep learning techniques to learn representations. We find that predicting adversarial value functions as auxiliary tasks leads to rich representations. %

\section{Setting}

We consider an environment described by a Markov Decision Process $\langle \cX, \cA, r, P, \gamma \rangle$ \citep{puterman94markov}; $\cX$ and $\cA$ are finite state and action spaces, $P : \cX \times \cA \to \dSet(\cX)$ is the transition function, $\gamma$ the discount factor, and $r : \cX \to \bR$ the reward function.
For a finite set $\cS$, write $\dSet(\cS)$ for the probability simplex over $\cS$.
A (stationary) \emph{policy} $\pi$ is a mapping $\cX \to \dSet(\cA)$, also denoted $\pi(a \cbar x)$. %
We denote the set of policies by $\SetPi = \dSet(\cA)^\cX$.
We combine a policy $\pi$ with the transition function $P$ to obtain the state-to-state transition function $\Ppi(x' \cbar x) := \sum_{a \in \cA} \pi(a \cbar x) P(x' \cbar x, a)$.
The \emph{value function} $\Vpi$ describes the expected discounted sum of rewards obtained by following $\pi$:
\begin{equation*}
\Vpi(x) = \expect \Big [ \sum_{t=0}^\infty \gamma^t r(x_t) \,\big |\, x_0 = x, x_{t+1} \sim \Ppi(\cdot \cbar x_t) \Big ].
\end{equation*}
The value function satisfies Bellman's equation \citep{bellman57dynamic}: $\Vpi(x) = r(x) + \gamma \expect_{\Ppi} \Vpi(x')$.
We will find it convenient to use vector notation: Assuming there are $n = |\cX|$ states, we view $r$ and $\Vpi$ as vectors in $\bR^n$ and $\Ppi \in \bR^{n \times n}$, yielding
\begin{equation*}
\Vpi = r + \gamma \Ppi \Vpi = (I - \gamma \Ppi)^{-1} r.
\end{equation*}

A \emph{$d$-dimensional representation} is a mapping $\phi : \cX \to \bR^d$; $\phi(x)$ is the \emph{feature vector} for state $x$. We write $\Phi \in \bR^{n \times d}$ to denote the matrix whose rows are $\phi(\cX)$, and with some abuse of notation denote the set of $d$-dimensional representations by $\SetPhi \equiv \bR^{n \times d}$.
For a given representation and weight vector $\theta \in \bR^d$, the linear approximation for a value function is
\begin{equation*}
\hat V_{\phi, \theta}(x) := \phi(x)^\top \theta .
\end{equation*}
We consider the approximation minimizing the uniformly weighted squared error
\begin{equation*}
\norm{\hat V_{\phi, \theta} - \Vpi}^2_2 = \sum_{x \in \cX} (\phi(x)^\top \theta - \Vpi(x))^2 .
\end{equation*}
We denote by $\hVphipi$ the projection of $\Vpi$ onto the linear subspace $H = \big \{ \Phi \theta : \theta \in \bR^d \big \}$.%

\subsection{Two-Part Approximation}

\begin{figure*}
\center{
\includegraphics[width=2.3in]{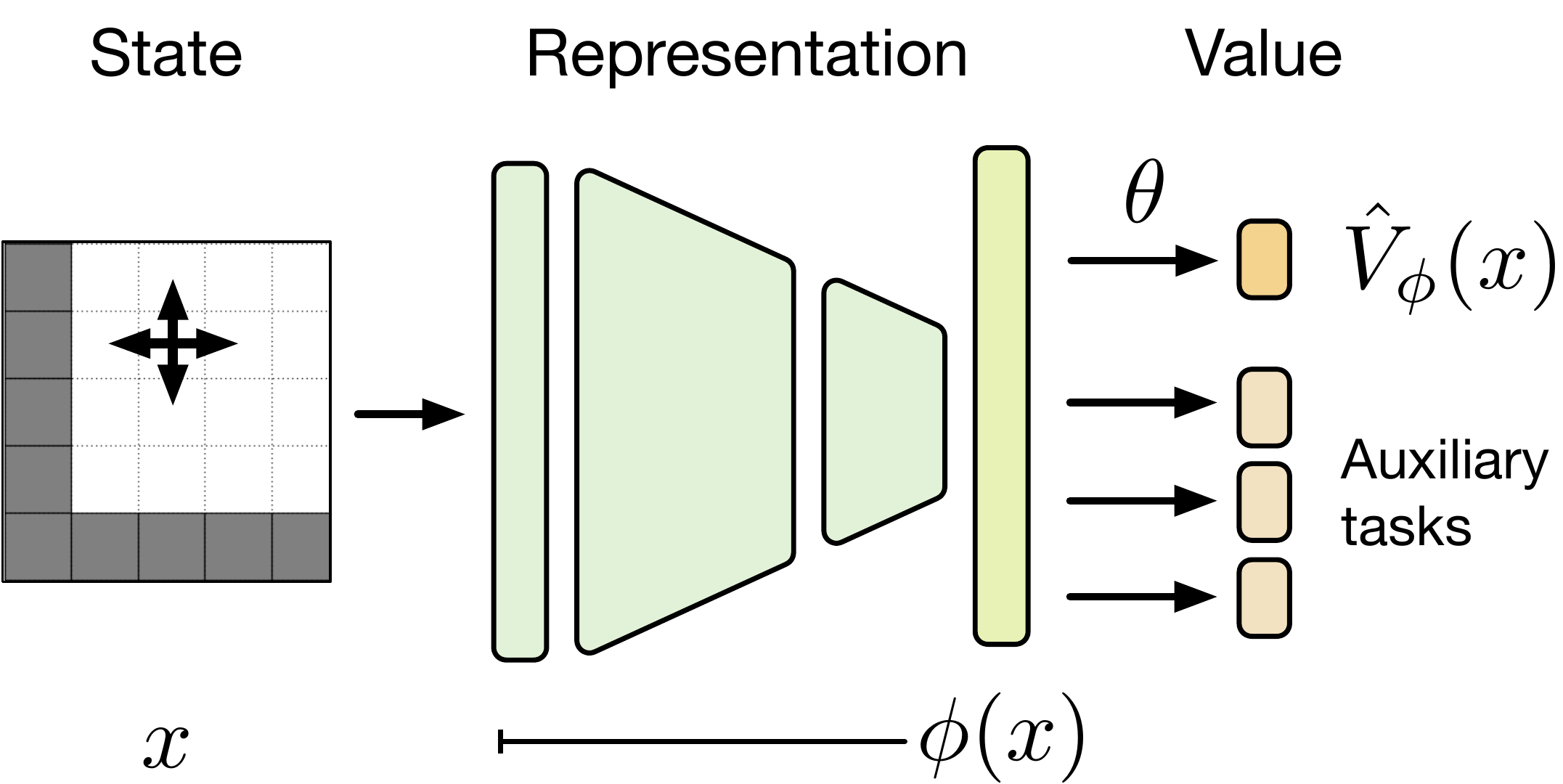}
\includegraphics[width=1.2in]{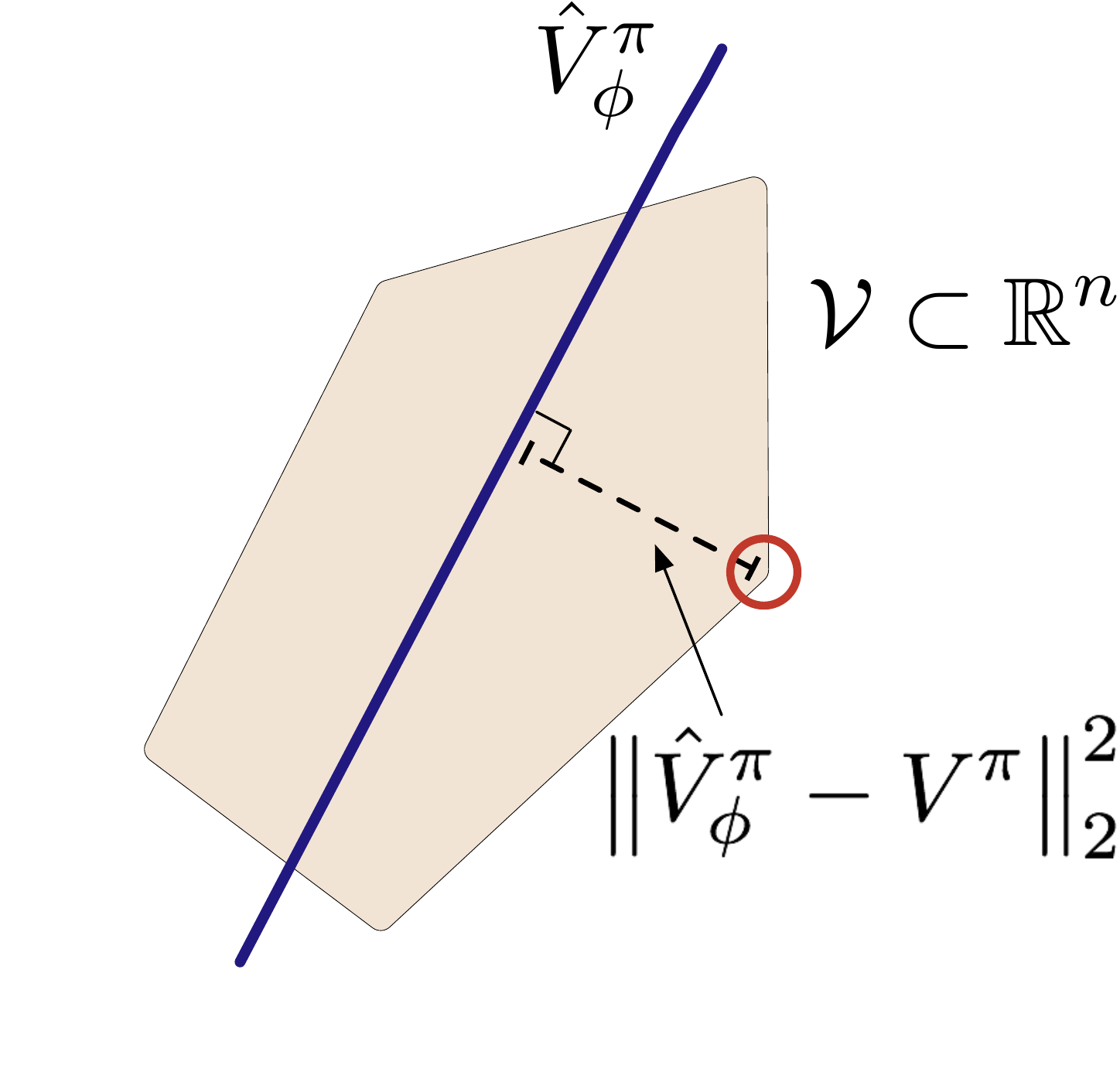}
\caption{\textbf{Left.} A deep reinforcement learning architecture viewed as a two-part approximation. \textbf{Right.}
The optimal representation $\phi^*$ is a linear subspace that cuts through the value polytope.\label{fig:geometric_perspective}}
}
\end{figure*}

We view $\hVphipi$ as a \emph{two-part approximation} arising from the composition of an adjustable representation $\phi$ and a weight vector $\theta$; we use the term ``two-part'' to emphasize that the mapping $\phi(x) \mapsto \hVphipi(x)$ is linear, while $\phi$ itself may not be. This separation into two parts gives us a simple framework in which to study the behaviour of representation learning, in particular deep networks applied to reinforcement learning. %
We will further consider the use of $\phi(x)$ to make additional predictions, called \emph{auxiliary tasks} following common usage, and whose purpose is to improve or stabilize the representation.

We study two-part approximations in an idealized setting where the length $d$ of $\phi(x)$ is fixed and smaller than $n$, but the mapping is otherwise unconstrained. Even this idealized design offers interesting problems to study. We might be interested in sharing a representation across problems, as is often done in transfer or continual learning. %
In this context, auxiliary tasks may inform how the value function should generalize to these new problems.
In many problems of interest, the weights $\theta$ can also be optimized more efficiently than the representation itself, warranting the view that the representation should be adapted using a different process \citep{levine17shallow,chung19twotimescale}.

Note that a trivial ``value-as-feature'' representation exists for the single-policy optimization problem
\begin{equation*}
\min_{\phi \in \SetPhi} \norm{\hVphipi - \Vpi}^2_2;
\end{equation*}
this approximation sets $\phi(x) = \Vpi(x), \theta = 1$. In this paper we take the stance that this is not a satisfying representation, and that a good representation should be in the service of a broader goal (e.g. control, transfer, or fairness).

\section{Representation Learning by Approximating Value Functions}

We measure the quality of a representation $\phi$ in terms of how well it can approximate all possible value functions, formalized as the \emph{representation error}
\begin{equation*}
L(\phi) := \max_{\pi \in \SetPi} L(\phi ; \pi), \quad L(\phi ; \pi) := \norm{\hVphipi - \Vpi}^2_2\, .
\end{equation*}
We consider the problem of finding the representation $\phi \in \SetPhi$ minimizing $L(\phi)$:
\begin{equation}\label{eqn:representation_learning_problem}
\min_{\phi \in \SetPhi} \max_{\pi \in \SetPi} \norm{\hVphipi - \Vpi}^2_2 .
\end{equation}
In the context of our work, we call this the \emph{representation learning problem} (\rlp) and say that a representation $\phi^*$ is \emph{optimal} when it minimizes the error in \eqnref{representation_learning_problem}.
Note that $L(\phi)$ (and hence $\phi^*$) depends on characteristics of the environment, in particular on both reward and transition functions.

We consider the \rlp~from a geometric perspective (Figure \ref{fig:geometric_perspective}, right). \citet{dadashi19value} showed  that the set of value functions achieved by the set of policies $\SetPi$, denoted
\begin{equation*}
\cV := \{ V \in \bR^n : V = \Vpi \text{ for some } \pi \in \SetPi \},
\end{equation*}
forms a (possibly nonconvex) polytope. 
As previously noted, a given representation $\phi$ defines a linear subspace $H$ of possible value approximations. 
The maximal error is achieved by the value function in $\cV$ which is furthest along the subspace normal to $H$, since $\hVphipi$ is the orthogonal projection of $\Vpi$.%

We say that $V \in \cV$ is an \emph{extremal vertex} if it is a vertex of the convex hull of $\cV$. Our first result shows that for any direction $\delta \in \bR^n$, the furthest point in $\cV$ along $\delta$ is an extremal vertex, and is in general unique for this $\delta$ (proof in the appendix).
\begin{lem}\label{lem:extremal_vertices}
Let $\delta \in \bR^n$ and define the functional $f_\delta(V) := \delta^\top V$, with domain $\cV$. Then $f_\delta$ is maximized by an extremal vertex $U \in \cV$, and there is a deterministic policy $\pi$ for which $V^\pi = U$. Furthermore, the set of directions $\delta \in \bR^n$ for which the maximum of $f_\delta$ is achieved by multiple extremal vertices has Lebesgue measure zero in $\bR^n$.
\end{lem}

Denote by $\SetPi_v$ the set of policies corresponding to extremal vertices of $\cV$.
We next derive an equivalence between the \rlp~and an optimization problem which only considers policies in $\SetPi_v$.
\begin{thm}\label{thm:equivalent_optimization_problems}
For any representation $\phi \in \SetPhi$, the maximal approximation error measured over all value functions is the same as the error measured over the set of extremal vertices:
\begin{equation*}
\max_{\pi \in \SetPi} \norm{\hVphipi - \Vpi}^2_2 = \max_{\pi \in \SetPi_v} \norm{\hVphipi - \Vpi}^2_2 .
\end{equation*}
\end{thm}
Theorem \ref{thm:equivalent_optimization_problems} indicates that
we can find an optimal representation by considering a finite (albeit exponential) number of value functions, since each extremal vertex corresponds to the value function of some deterministic policy, of which there are at most an exponential number.
 We will call these \emph{adversarial value functions} (AVFs), because of the minimax flavour of the \rlp.

Solving the \rlp~allows us to provide quantifiable guarantees on the performance of certain value-based learning algorithms. For example, in the context of least-squares policy iteration \citep[LSPI;][]{lagoudakis03leastsquares}, minimizing the representation error $L$ directly improves the performance bound. By contrast, we cannot have the same guarantee if $\phi$ is learned by minimizing the approximation error for a single value function.
\begin{cor}\label{cor:lspi_bound}
Let $\phi^*$ be an optimal representation in the \rlp. Consider the sequence of policies $\pi_0, \pi_1, \dots$ derived from LSPI using $\phi^*$ to approximate $V^{\pi_0}, V^{\pi_1}, \dots$ under a uniform sampling of the state-space. Then there exists an MDP-dependent constant $C \in \bR$ such that
\begin{equation*}
\limsup_{k \to \infty} \norm{V^* - V^{\pi_k}}^2_2 \le C L(\phi^*).
\end{equation*}
\end{cor}
\vspace{-0.8em}
This result is a direct application of the quadratic norm bounds given by \citet{munos03error}, in whose work the constant is made explicit. We emphasize that the result is illustrative; our approach should enable similar guarantees in other contexts \citep[e.g.][]{munos07performance,petrik11robust}.

\subsection{The Structure of Adversarial Value Functions}\label{subsec:adversarial_vf_discovery}

The \rlp~suggests that an agent trained to predict various value functions should develop a good state representation.
Intuitively, one may worry that there are simply too many ``uninteresting'' policies, and that a representation learned from their value functions emphasizes the wrong quantities. However, the search for an optimal representation $\phi^*$ is closely tied to the much smaller set of adversarial value functions (AVFs).
The aim of this section is to characterize the structure of AVFs and show that they form an \emph{interesting} subset of all value functions. From this, we argue that their use as auxiliary tasks should also produce structured representations. %

From Lemma \ref{lem:extremal_vertices}, recall that an AVF is geometrically defined using a vector $\delta \in \bR^n$ and the functional $f_\delta(V) := \delta^\top V$, which the AVF maximizes. Since $f_\delta$ is restricted to the value polytope, we can consider the equivalent policy-space functional $g_\delta : \pi \mapsto \delta^\top V^\pi$. Observe that
\begin{equation}\label{eqn:delta_optimization_problem}
\max_{\pi \in \SetPi} g_\delta(\pi) = \max_{\pi \in \SetPi} \delta^\top \Vpi = \max_{\pi \in \SetPi} \sum_{x \in \cX} \delta(x) \Vpi(x) .
\end{equation}
In this optimization problem, the vector $\delta$ defines a weighting over the state space $\cX$; for this reason, we call $\delta$ an \emph{interest function} in the context of AVFs. Whenever $\delta \ge 0$ componentwise, we recover the optimal  value function, irrespective of the exact magnitude of $\delta$ \citep{bertsekas12dynamic}.
If $\delta(x) < 0$ for some $x$, however, the maximization becomes a minimization. As the next result shows, the policy maximizing $f_\delta(\pi)$ depends on a network flow $\dpi$ derived from $\delta$ and the transition function $P$.
\begin{thm}\label{thm:dual_delta_method}
Maximizing the functional $g_\delta$ is equivalent to finding a network flow $\dpi$ that satisfies a reverse Bellman equation:
\begin{equation*}
\max_{\pi \in \SetPi} \delta^\top \Vpi = \max_{\pi \in \SetPi} \dpi^\top r, \qquad \dpi = \delta + \gamma \Ppi^\top \dpi .
\end{equation*}
For a policy $\tpi$ maximizing the above we have
\begin{equation*}
\Vtpi(x) = r(x) + \gamma \left \{ \begin{array}{ll} 
	\max_{a \in \cA} \expect_{x' \sim P} \Vtpi(x') & d_{\tilde \pi}(x) > 0, \\
	\min_{a \in \cA} \expect_{x' \sim P} \Vtpi(x') & d_{\tilde \pi}(x) < 0 . \\
	\end{array} \right .
\end{equation*}
\end{thm}
\begin{cor}\label{cor:2n_distinct_avfs}
There are at most $2^n$ distinct adversarial value functions.
\end{cor}
The vector $\dpi$ corresponds to the sum of discounted interest weights flowing through a state $x$, similar to the dual variables in the theory of linear programming for MDPs \citep{puterman94markov}.
Theorem \ref{thm:dual_delta_method}, by way of the corollary, implies that there are fewer AVFs ($\le 2^n$) than deterministic policies ($= |\cA|^n)$. It also implies that AVFs relate to a reward-driven purpose, similar to how the optimal value function describes the goal of maximizing return. We will illustrate this point empirically in Section \ref{subsec:empirical_studies_avfs}. %

\subsection{Relationship to Auxiliary Tasks}\label{sec:relationship_to_auxiliary_tasks}

So far we have argued that solving the \rlp~leads to a representation which is optimal in a meaningful sense. However, solving the \rlp~seems computationally intractable: there are an exponential number of deterministic policies to consider (Prop. \ref{prop:optimization} in the appendix gives a quadratic formulation with quadratic constraints). Using interest functions does not mitigate this difficulty: the computational problem of finding the AVF for a single interest function is NP-hard, even when restricted to deterministic MDPs (Prop. \ref{prop:np-hard} in the appendix).

Instead, in this section we consider a relaxation of the \rlp~and show that this relaxation describes existing representation learning methods, in particular those that use auxiliary tasks. Let $\xi$ be some distribution over $\bR^n$. We begin by replacing the maximum in \eqnref{representation_learning_problem} by an expectation:
\begin{equation}\label{eqn:relaxed_learning_problem}
\min_{\phi \in \SetPhi} \expect_{V \sim \xi} \norm{\hat V_\phi - V}_2^2.
\end{equation}

The use of the expectation offers three practical advantages over the use of the maximum. First, this leads to a differentiable objective which can be minimized using deep learning techniques. Second, the choice of $\xi$ gives us an additional degree of freedom; in particular, $\xi$ needs not be restricted to the value polytope. Third, the minimizer in \eqnref{relaxed_learning_problem} is easily characterized, as the following theorem shows.

\begin{thm}\label{thm:optimal_representation_is_svd}
Let $u^*_1, \dots, u^*_d \in \bR^n$ be the principal components of the distribution $\xi$, in the sense that
\begin{equation*}
u^*_i := \argmax_{u \in B_i} \expect_{V \sim \xi} (u^\top V)^2, \text{ where } B_i := \{ u \in \bR^n : \| u \|^2_2 = 1, u^\top u^*_j = 0 \; \forall j < i \} .
\end{equation*}
Equivalently, $u^*_1, \dots, u^*_d$ are the eigenvectors of $\expect\nolimits_{\xi} VV^\top \in \bR^{n \times n}$ with the $d$ largest eigenvalues.
Then the matrix $[u_1^*, \dots, u_d^*] \in \bR^{n \times d}$, viewed as a map $\cX \to \bR^d$, is a solution to \eqnref{relaxed_learning_problem}. When the principal components are uniquely defined, any minimizer of \eqnref{relaxed_learning_problem} spans the same subspace as $u_1^*, \dots, u_d^*$.
\end{thm}

One may expect the quality of the learned representation to depend on how closely the distribution $\xi$ relates to the \rlp. 
From an auxiliary tasks perspective, this corresponds to choosing tasks that are in some sense useful.
For example, generating value functions from the uniform distribution over the set of policies $\SetPi$, while a natural choice, may put too much weight on ``uninteresting'' value functions. 

In practice, we may further restrict $\xi$ to a finite set $\bbV$. Under a uniform weighting, this leads to a \emph{representation loss}
\begin{equation}\label{eqn:auxiliary_tasks_loss}
L(\phi; \bbV) := \sum_{V \in \bbV} \norm{\hat V_\phi - V}_2^2
\end{equation}
which corresponds to the typical formulation of an auxiliary-task loss \citep[e.g.][]{jaderberg17reinforcement}.
In a deep reinforcement learning setting, one typically minimizes \eqnref{auxiliary_tasks_loss} using stochastic gradient descent methods, which scale better than batch methods such as singular value decomposition (but see \citet{wu19laplacian} for further discussion). 

Our analysis leads us to conclude that, in many cases of interest, the use of auxiliary tasks produces representations that are close to the principal components of the set of tasks under consideration. If $\bbV$ is well-aligned with the \rlp, minimizing $L(\phi; \bbV)$
should give rise to a reasonable representation.
To demonstrate the power of this approach, in Section \ref{sec:empirical_studies} we will study the case when the set $\bbV$ is constructed by sampling AVFs -- emphasizing the policies that support the solution to the \rlp.

\subsection{Relationship to Proto-Value Functions}

Proto-value functions \citep[\textsc{pvf}]{mahadevan07proto} are a family of representations which vary smoothly across the state space. 
Although the original formulation defines this representation as the largest-eigenvalue eigenvectors of the Laplacian of the transition function's graphical structure, recent formulations use the top singular vectors of $(I - \gamma \Ppi)^{-1}$, where $\pi$ is the uniformly random policy \citep{stachenfeld14design,machado17laplacian,behzadian18feature}.

In line with the analysis of the previous section, proto-value functions can also be interpreted as defining a set of value-based auxiliary tasks. Specifically, if we define an indicator reward function $r_{y}(x) := \indic{x = y}$ and a set of value functions $\bbV = \{ (I - \gamma P^\pi)^{-1} r_y \}_{y \in \cX}$ with $\pi$ the uniformly random policy, then any $d$-dimensional representation that minimizes \eqnref{auxiliary_tasks_loss} spans the same basis as the $d$-dimensional \textsc{pvf} (up to the bias term). This suggests a connection with hindsight experience replay \citep{andrychowicz17hindsight}, whose auxiliary tasks consists in reaching previously experienced states.

\section{Empirical Studies}\label{sec:empirical_studies}

In this section we complement our theoretical analysis with an experimental study. In turn, we take a closer look at
1) the \textbf{structure} of adversarial value functions, 2) the \textbf{shape of representations} learned using AVFs, and 
3) the \textbf{performance profile} of these representations in a control setting.

Our eventual goal is to demonstrate that the representation learning problem \eqnref{representation_learning_problem}, which is based on approximating value functions, gives rise to representations that are both interesting and comparable to previously proposed schemes. Our concrete instantiation (Algorithm \ref{alg:AVF_representation_learning}) uses the representation loss \eqnref{auxiliary_tasks_loss}. As-is, this algorithm is of limited practical relevance (our AVFs are learned using a tabular representation) but we believe provides an inspirational basis for further developments.

\begin{algorithm}[ht]
\caption{Representation learning using AVFs}\label{alg:AVF_representation_learning}
\begin{algorithmic}
\INPUT $k$ -- desired number of AVFs, $d$ -- desired number of features.
\STATE Sample $\delta_1, \dots, \delta_k \sim [-1, 1]^n$
\STATE Compute $\mu_i = \argmax_{\pi} \delta_i^\top \Vpi$ using a policy gradient method
\STATE Find $\phi^* = \argmin_{\phi} L(\phi; \{V^{\mu_1}, \dots, V^{\mu_k} \})$ (Equation \ref{eqn:auxiliary_tasks_loss})
\end{algorithmic}
\end{algorithm}

We perform all of our experiments within the four-room domain
\citep[][Figure \ref{fig:alignment-and-AVF}, see also Appendix \ref{sec:four_room_domain}]{sutton99between,solway14optimal,machado17laplacian}.

\begin{figure*}[!ht]
\center{
\includegraphics[width=1.1in]{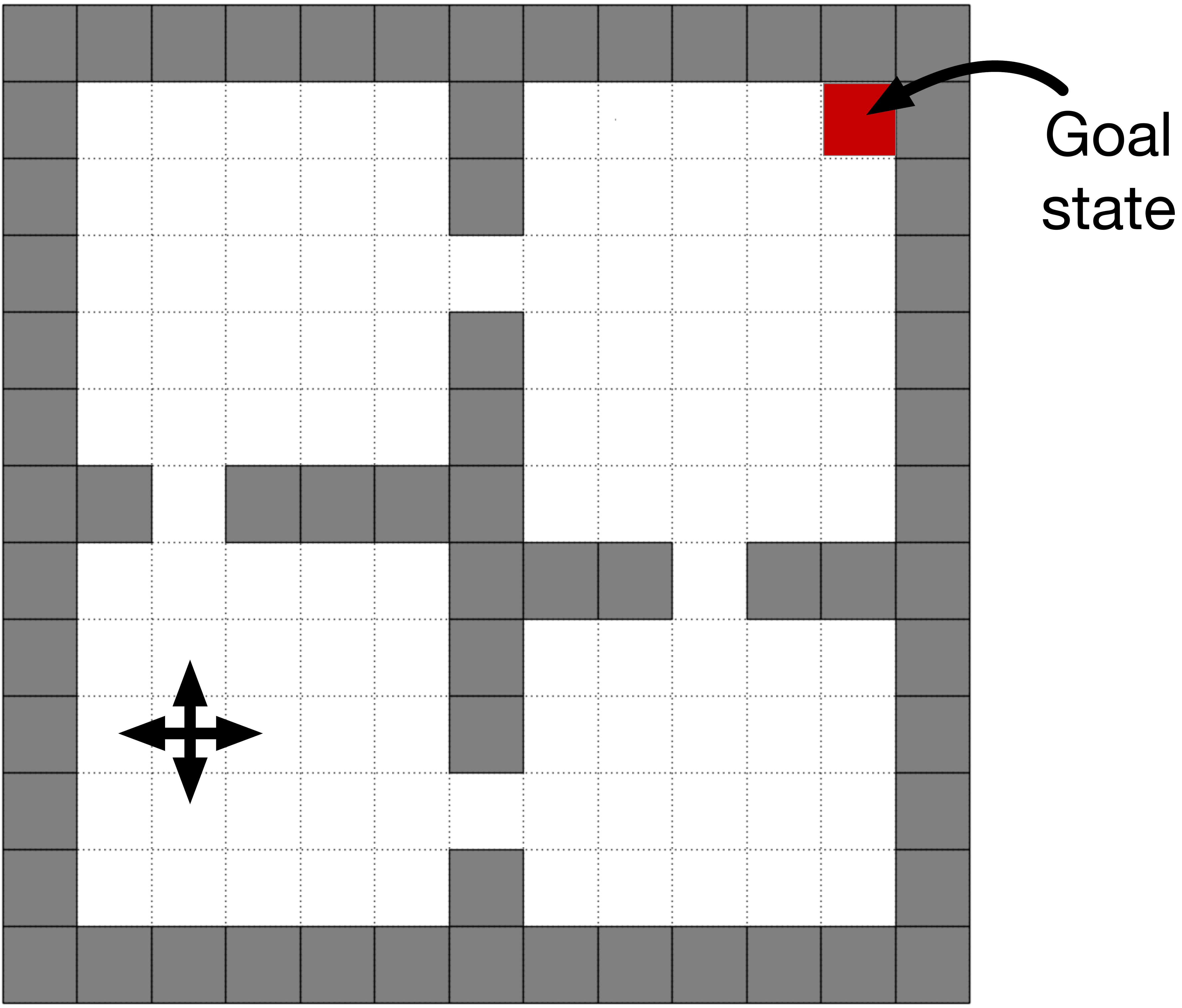}
\includegraphics[width=4.2in]{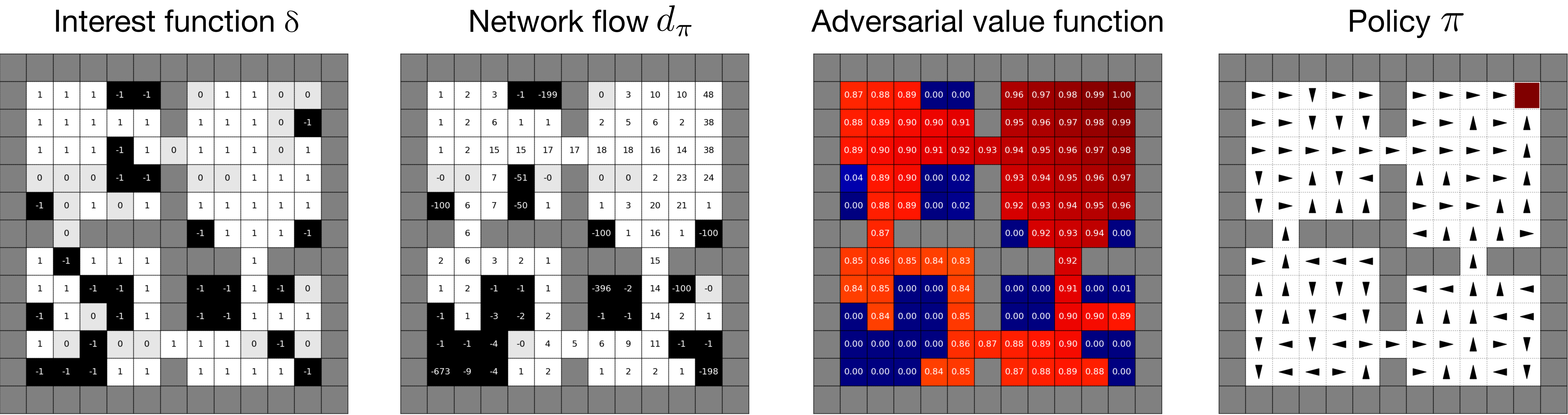}
\caption{\textbf{Leftmost.} The four-room domain. \textbf{Other panels.} An interest function $\delta$, the network flow $\dpi$, the corresponding adversarial value function (blue/red = low/high value) and its policy.
\label{fig:alignment-and-AVF}}
}
\end{figure*}

We consider a two-part approximation where we pretrain $\phi$ end-to-end to predict a set of value functions. Our aim here is to compare the effects of using different sets of value functions, including AVFs, on the learned representation. %
As our focus is on the efficient use of a $d$-dimensional representation (with $d < n$, the number of states), we encode individual states as one-hot vectors and map them into $\phi(x)$ without capacity constraints. 
Additional details %
may be found in Appendix \ref{sec:empirical_studies_methodology}.

\subsection{Adversarial Value Functions}\label{subsec:empirical_studies_avfs}

Our first set of results studies the structure of adversarial value functions in the four-room domain. We generated interest functions by assigning a value $\delta(x) \in \{-1, 0, 1\}$ uniformly at random to each state $x$ (Figure \ref{fig:alignment-and-AVF}, left). We restricted $\delta$ to these discrete choices for illustrative purposes.

We then used model-based policy gradient \citep{sutton00policy} to find the policy maximizing $\sum_{x \in \cX} \delta(x) V^\pi(x)$. %
We observed some local minima or accumulation points but as a whole reasonable solutions were found. The resulting network flow and AVF for a particular sample are shown in Figure \ref{fig:alignment-and-AVF}.
For most states, the signs of $\delta$ and $\dpi$ agree; however, this is not true of all states (larger version and more examples in appendix, Figures \ref{fig:alignment-and-AVF-big}, \ref{fig:more-interest-functions}).
As expected, states for which $\dpi > 0$ (respectively, $\dpi < 0$) correspond to states maximizing (resp. minimizing) the value function. %
Finally, we remark on the ``flow'' nature of $\dpi$: trajectories over minimizing states accumulate in corners or loops, while those over maximizing states flow to the goal.
We conclude that AVFs exhibit interesting structure, and are generated by policies that are not random (Figure \ref{fig:alignment-and-AVF}, right). As we will see next, this is a key differentiator in making AVFs good auxiliary tasks.

\subsection{Representation Learning with AVFs}

We next consider the representations that arise from training a deep network to predict AVFs (denoted \textsc{avf} from here on).
We sample $k = 1000$ interest functions and use Algorithm \ref{alg:AVF_representation_learning} to generate $k$ AVFs. 
We combine these AVFs into the representation loss \eqnref{auxiliary_tasks_loss} and adapt the parameters of the deep network using Rmsprop \citep{tieleman12rmsprop}.

We contrast the AVF-driven representation with one learned by predicting the value function of random deterministic policies (\textsc{rp}). Specifically, these policies are generated by assigning an action uniformly at random to each state. We also consider the value function of the uniformly random policy (\textsc{value}). While we make these choices here for concreteness, other experiments yielded similar results (e.g. predicting the value of the optimal policy; appendix, Figure \ref{fig:learned-value-representations}). In all cases, we learn a $d=16$ dimensional representation, not including the bias unit.
\begin{figure*}[htb]
\center{
\includegraphics[width=4.6in]{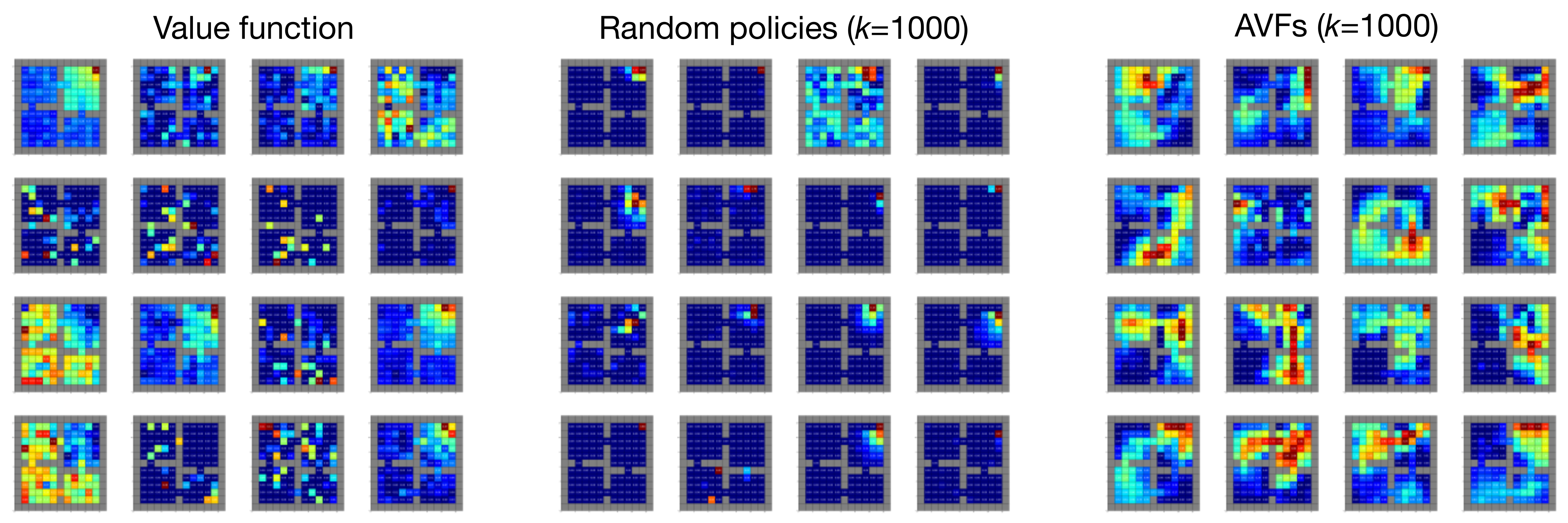}
\caption{16-dimensional representations learned by predicting a single value function, the value functions of 1000 random policies, or 1000 AVFs sampled using Algorithm \ref{alg:AVF_representation_learning}. Each panel element depicts the activation of a given feature across states, with blue/red indicating low/high activation.\label{fig:different-representations}}
}
\end{figure*}

Figure \ref{fig:different-representations} shows the representations learned by the three methods. The features learned by \textsc{value} resemble the value function itself (top left feature) or its negated image (bottom left feature). Coarsely speaking, these features capture the general distance to the goal but little else. The features learned by \textsc{rp} are of even worse quality. This is because almost all random deterministic policies cause the agent to avoid the goal (appendix, Figure \ref{fig:random-policies-vfs}). 
The representation learned by \textsc{avf}, on the other hand, captures the structure of the domain, including paths between distal states and focal points corresponding to rooms or parts of rooms.

Although our focus is on the use of AVFs as auxiliary tasks to a deep network, we observe the same results when discovering a representation using singular value decomposition (Section \ref{sec:relationship_to_auxiliary_tasks}), as described in Appendix \ref{sec:representations_from_svd}.
All in all, our results illustrate that, among all value functions, AVFs are particularly useful auxiliary tasks for representation learning.

\subsection{Learning the Optimal Policy}
\begin{wrapfigure}{r}{0.5 \textwidth}
\vspace{-1em}
\center{
\includegraphics[width=.42\textwidth]{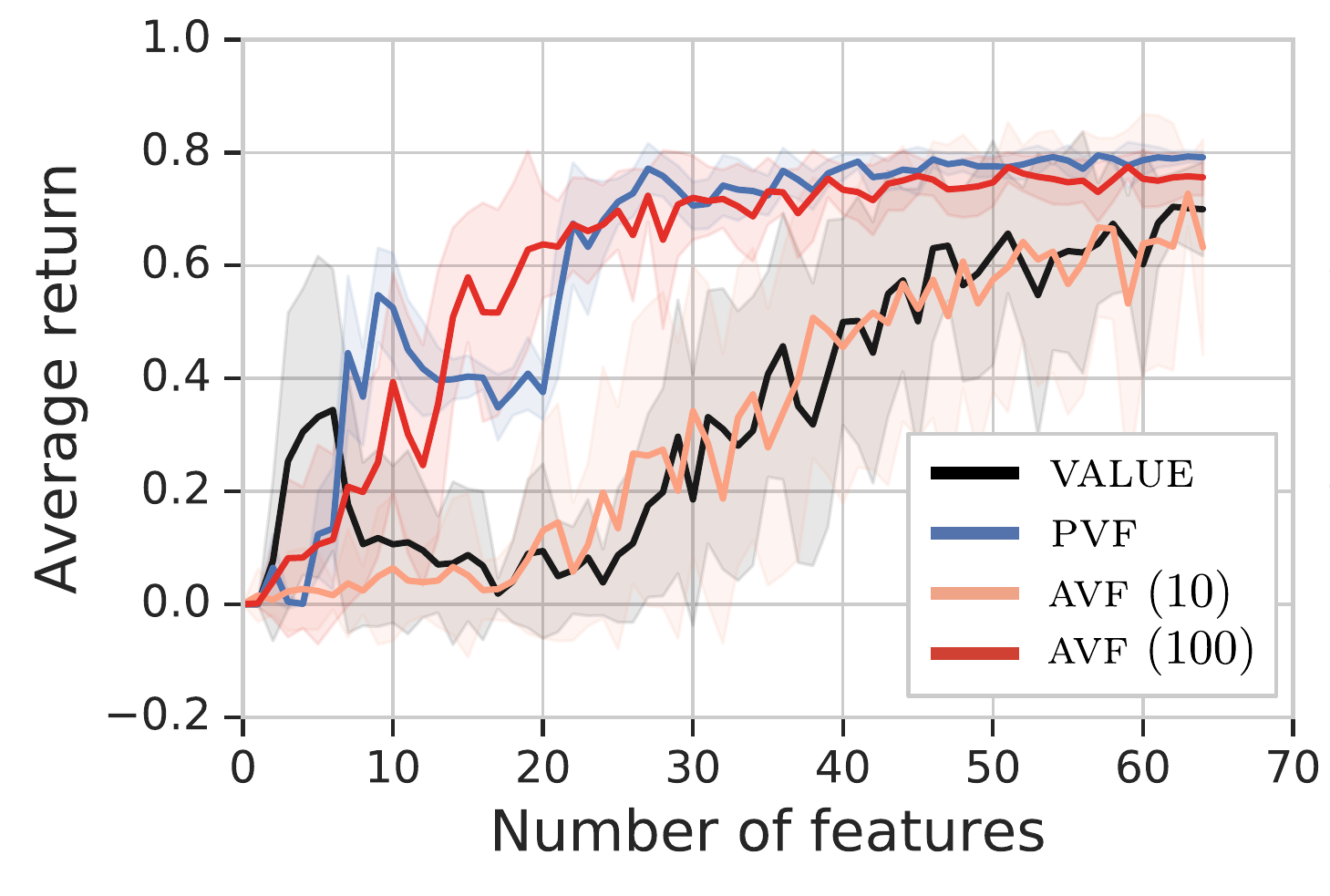}
\caption{Average discounted return achieved by policies learned using a representation produced by $\textsc{value}$, $\textsc{avf}$, or $\textsc{pvf}$. Average is over $20$ random seeds and shading gives standard deviation.\label{fig:sarsa-performance}}
}
\end{wrapfigure}

In a final set of experiments, we consider learning a reward-maximizing policy using a pretrained representation and
a model-based version of the SARSA algorithm \citep{rummery94online,sutton98reinforcement}. 
We compare the value-based and AVF-based representations from the previous section (\textsc{value} and {\textsc{avf}), and also proto-value functions (\textsc{pvf}; details in Appendix \ref{sec:pvf_implementation}). 

We report the quality of the learned policies after training, as a function of $d$, the size of the representation. Our quality measures is the average return from the designated start state (bottom left). Results are provided in Figure \ref{fig:sarsa-performance} and Figure \ref{fig:sarsa_both} (appendix).
We observe a failure of the \textsc{value} representation to provide a useful basis for learning a good policy, even as $d$ increases; while the representation is not rank-deficient, the features do not help reduce the approximation error. In comparison, our \textsc{avf} representations perform similarly to \textsc{pvf}s. Increasing the number of auxiliary tasks also leads to better representations; recall that \textsc{pvf} implicitly uses $n=104$ auxiliary tasks.

\section{Related Work}

Our work takes inspiration from earlier research in basis or feature construction for reinforcement learning. \citet{ratitch04sparse},
\citet{foster02structure}, \citet{menache05basis}, \citet{yu09basis}, \citet{bhatnagar13feature}, and \citet{song16linear} consider methods for adapting parametrized basis functions using iterative schemes. 
Including \citet{mahadevan07proto}'s proto-value functions, a number of works \citep[we note][]{dayan93improving,petrik07analysis,mahadevan10basis,ruan15representation,barreto17successor} have used characteristics of the transition structure of the MDP to generate representations; these are the closest in spirit to our approach, although none use the reward or consider the geometry of the space of value functions. \citet{parr07analyzing}
proposed constructing a representation from successive Bellman errors, \citet{keller06automatic} used dimensionality reduction methods; finally \citet{hutter09feature} proposes a universal scheme for selecting representations.

Deep reinforcement learning algorithms have made extensive use of auxiliary tasks to improve agent performance, beginning perhaps with universal value function approximators \citep{schaul15universal} and the UNREAL architecture \citep{jaderberg17reinforcement}; see also \citet{dosovitskiy17learning}, \citet{francoislavet18combined} and, more tangentially, \citet{vandenoord18representation}. \citet{levine17shallow} and \citet{chung19twotimescale} make explicit use of two-part approximations to derive more sample efficient deep reinforcement learning algorithms.
The notion of augmenting an agent with side predictions regarding the world is not new, with roots in TD models \citep{sutton95td}, predictive state representations \citep{littman02predictive} and the Horde architecture \citep{sutton11horde}, which itself is based on \citet{selfridge59pandemonium}'s Pandemonium architecture.

In passing, we remark on a number of works which aim to quantify or explain the usefulness of a representation. \citet{parr08analysis} studies the particular case of linear representations, while \citet{li06towards,abel16nearoptimal} consider the approximation error that arises from state abstraction. More recently, \citet{nachum19nearoptimal} provide some interesting guarantees in the context of hierarchical reinforcement learning, while \citet{such19atari} visualizes the representations learned by Atari-playing agents. Finally, \citet{bertsekas18featurebased} remarks on the two-part approximation we study here.

\section{Conclusion}

In this paper we studied the notion of an adversarial value function, derived from a geometric perspective on representation learning in RL. Our work shows that adversarial value functions exhibit interesting structure, and are good auxiliary tasks when learning a representation of an environment. We believe our work to be the first to provide formal evidence as to the usefulness of predicting value functions for shaping an agent's representation.

Our work opens up the possibility of automatically generating auxiliary tasks in deep reinforcement learning, analogous to how deep learning itself enabled a move away from hand-crafted features. %
A number of practical considerations remain to be addressed. First, our sampling procedure is clearly inefficient, and may be improved by encouraging diversity within AVFs. Second, practical implementations require learning the AVFs concurrently with the main task. Doing results in off-policy learning, whose negative effects are well-documented even in recent applications. \citep[e.g.][]{vanhasselt18deep}. Third, interest functions in large domains should incorporate some degree of smoothness, rather than vary rapidly from state to state. %

From a mathematical perspective, our formulation of the representation learning problem \eqnref{representation_learning_problem} was made with both convenience and geometry in mind. Conceptually, it may be interesting to consider our approach in other norms, including the weighted norms used in approximation results.

\section{Acknowledgements}

The authors thank the many people who helped shape this project through discussions and feedback on early drafts: Lihong Li, George Tucker, Doina Precup, Ofir Nachum, Csaba Szepesv\'ari, Georg Ostrovski, Marek Petrik, Marlos Machado, Tim Lillicrap, Danny Tarlow, Saurabh Kumar, and Carles Gelada. Special thanks also to Philip Thomas and Scott Niekum, who gave this project its initial impetus.

\section{Author Contributions}

M.G.B., W.D., D.S., and N.L.R. conceptualized the representation learning problem. M.G.B., W.D., T.L., A.A.T., R.D., D.S., and N.L.R. contributed to the theoretical results. M.G.B., W.D., P.S.C., R.D., and C.L. performed experiments and collated results. All authors contributed to the writing.

\bibliographystyle{apalike}
\bibliography{auxiliary-tasks}

\begin{thebibliography}{}

\bibitem[Abadi et~al., 2016]{abadi16tensorflow}
Abadi, M., Barham, P., Chen, J., Chen, Z., Davis, A., Dean, J., Devin, M.,
  Ghemawat, S., Irving, G., Isard, M., et~al. (2016).
\newblock Tensorflow: {A} system for large-scale machine learning.
\newblock In {\em Symposium on Operating Systems Design and Implementation}.

\bibitem[Abel et~al., 2016]{abel16nearoptimal}
Abel, D., Hershkowitz, D.~E., and Littman, M.~L. (2016).
\newblock Near optimal behavior via approximate state abstraction.
\newblock In {\em Proceedings of the International Conference on Machine
  Learning}.

\bibitem[Andrychowicz et~al., 2017]{andrychowicz17hindsight}
Andrychowicz, M., Wolski, F., Ray, A., Schneider, J., Fong, R., Welinder, P.,
  McGrew, B., Tobin, J., Abbeel, O.~P., and Zaremba, W. (2017).
\newblock Hindsight experience replay.
\newblock In {\em Advances in Neural Information Processing Systems}.

\bibitem[Barreto et~al., 2017]{barreto17successor}
Barreto, A., Dabney, W., Munos, R., Hunt, J.~J., Schaul, T., van Hasselt,
  H.~P., and Silver, D. (2017).
\newblock Successor features for transfer in reinforcement learning.
\newblock In {\em Advances in Neural Information Processing Systems}.

\bibitem[Behzadian and Petrik, 2018]{behzadian18feature}
Behzadian, B. and Petrik, M. (2018).
\newblock Feature selection by singular value decomposition for reinforcement
  learning.
\newblock In {\em Proceedings of the ICML Prediction and Generative Modeling
  Workshop}.

\bibitem[Bellemare et~al., 2017]{bellemare17distributional}
Bellemare, M.~G., Dabney, W., and Munos, R. (2017).
\newblock A distributional perspective on reinforcement learning.
\newblock In {\em Proceedings of the International Conference on Machine
  Learning}.

\bibitem[Bellemare et~al., 2013]{bellemare13arcade}
Bellemare, M.~G., Naddaf, Y., Veness, J., and Bowling, M. (2013).
\newblock The {A}rcade {L}earning {E}nvironment: An evaluation platform for
  general agents.
\newblock {\em Journal of Artificial Intelligence Research}, 47:253--279.

\bibitem[Bellman, 1957]{bellman57dynamic}
Bellman, R.~E. (1957).
\newblock {\em Dynamic programming}.
\newblock Princeton University Press, Princeton, NJ.

\bibitem[Bernhard and Vygen, 2008]{BV08}
Bernhard, K. and Vygen, J. (2008).
\newblock Combinatorial optimization: Theory and algorithms.
\newblock {\em Springer, Third Edition, 2005.}

\bibitem[Bertsekas, 2012]{bertsekas12dynamic}
Bertsekas, D.~P. (2012).
\newblock {\em Dynamic Programming and Optimal Control, Vol. II: Approximate
  Dynamic Programming}.
\newblock Athena Scientific.

\bibitem[Bertsekas, 2018]{bertsekas18featurebased}
Bertsekas, D.~P. (2018).
\newblock Feature-based aggregation and deep reinforcement learning: {A} survey
  and some new implementations.
\newblock Technical report, MIT/LIDS.

\bibitem[Bhatnagar et~al., 2013]{bhatnagar13feature}
Bhatnagar, S., Borkar, V.~S., and Prabuchandran, K. (2013).
\newblock Feature search in the {G}rassmanian in online reinforcement learning.
\newblock {\em IEEE Journal of Selected Topics in Signal Processing}.

\bibitem[Boyd and Vandenberghe, 2004]{boyd04convex}
Boyd, S. and Vandenberghe, L. (2004).
\newblock {\em Convex optimization}.
\newblock Cambridge university press.

\bibitem[Castro et~al., 2018]{castro18dopamine}
Castro, P.~S., Moitra, S., Gelada, C., Kumar, S., and Bellemare, M.~G. (2018).
\newblock Dopamine: A research framework for deep reinforcement learning.
\newblock {\em arXiv}.

\bibitem[Chung et~al., 2019]{chung19twotimescale}
Chung, W., Nath, S., Joseph, A.~G., and White, M. (2019).
\newblock Two-timescale networks for nonlinear value function approximation.
\newblock In {\em International Conference on Learning Representations}.

\bibitem[Dadashi et~al., 2019]{dadashi19value}
Dadashi, R., Ta{\"i}ga, A.~A., Roux, N.~L., Schuurmans, D., and Bellemare,
  M.~G. (2019).
\newblock The value function polytope in reinforcement learning.
\newblock {\em arXiv}.

\bibitem[Dayan, 1993]{dayan93improving}
Dayan, P. (1993).
\newblock Improving generalisation for temporal difference learning: The
  successor representation.
\newblock {\em Neural Computation}.

\bibitem[Dosovitskiy and Koltun, 2017]{dosovitskiy17learning}
Dosovitskiy, A. and Koltun, V. (2017).
\newblock Learning to act by predicting the future.
\newblock In {\em Proceedings of the International Conference on Learning
  Representations}.

\bibitem[Ernst et~al., 2005]{ernst05treebased}
Ernst, D., Geurts, P., and Wehenkel, L. (2005).
\newblock Tree-based batch mode reinforcement learning.
\newblock {\em Journal of Machine Learning Research}, 6:503--556.

\bibitem[Farahmand et~al., 2016]{farahmand16regularized}
Farahmand, A., Ghavamzadeh, M., Szepesv{\'a}ri, C., and Mannor, S. (2016).
\newblock Regularized policy iteration with nonparametric function spaces.
\newblock {\em Journal of Machine Learning Research}.

\bibitem[Foster and Dayan, 2002]{foster02structure}
Foster, D. and Dayan, P. (2002).
\newblock Structure in the space of value functions.
\newblock {\em Machine Learning}.

\bibitem[Fran{\c c}ois-Lavet et~al., 2018]{francoislavet18combined}
Fran{\c c}ois-Lavet, V., Bengio, Y., Precup, D., and Pineau, J. (2018).
\newblock Combined reinforcement learning via abstract representations.
\newblock {\em arXiv}.

\bibitem[Gelada et~al., 2019]{gelada19deepmdp}
Gelada, C., Kumar, S., Buckman, J., Nachum, O., and Bellemare, M.~G. (2019).
\newblock Deep{MDP}: {L}earning continuous latent space modelsfor
  representation learning.
\newblock In {\em Proceedings of the International Conference on Machine
  Learning}.

\bibitem[Hutter, 2009]{hutter09feature}
Hutter, M. (2009).
\newblock Feature reinforcement learning: {P}art {I}. {U}nstructured {MDP}s.
\newblock {\em Journal of Artificial General Intelligence}.

\bibitem[Jaderberg et~al., 2017]{jaderberg17reinforcement}
Jaderberg, M., Mnih, V., Czarnecki, W.~M., Schaul, T., Leibo, J.~Z., Silver,
  D., and Kavukcuoglu, K. (2017).
\newblock Reinforcement learning with unsupervised auxiliary tasks.
\newblock In {\em Proceedings of the International Conference on Learning
  Representations}.

\bibitem[Keller et~al., 2006]{keller06automatic}
Keller, P.~W., Mannor, S., and Precup, D. (2006).
\newblock Automatic basis function construction for approximate dynamic
  programming and reinforcement learning.
\newblock In {\em Proceedings of the International Conference on Machine
  Learning}.

\bibitem[Lagoudakis and Parr, 2003]{lagoudakis03leastsquares}
Lagoudakis, M. and Parr, R. (2003).
\newblock Least-squares policy iteration.
\newblock {\em The Journal of Machine Learning Research}.

\bibitem[Levine et~al., 2017]{levine17shallow}
Levine, N., Zahavy, T., Mankowitz, D., Tamar, A., and Mannor, S. (2017).
\newblock Shallow updates for deep reinforcement learning.
\newblock In {\em Advances in Neural Information Processing Systems}.

\bibitem[Li et~al., 2006]{li06towards}
Li, L., Walsh, T., and Littman, M. (2006).
\newblock {Towards a unified theory of state abstraction for MDPs}.
\newblock In {\em Proceedings of the Ninth International Symposium on
  Artificial Intelligence and Mathematics}.

\bibitem[Liang et~al., 2016]{liang16state}
Liang, Y., Machado, M.~C., Talvitie, E., and Bowling, M.~H. (2016).
\newblock State of the art control of atari games using shallow reinforcement
  learning.
\newblock In {\em Proceedings of the International Conference on Autonomous
  Agents and Multiagent Systems}.

\bibitem[Littman et~al., 2002]{littman02predictive}
Littman, M.~L., Sutton, R.~S., and Singh, S. (2002).
\newblock Predictive representations of state.
\newblock In {\em Advances in Neural Information Processing Systems}.

\bibitem[Machado et~al., 2017]{machado17laplacian}
Machado, M.~C., Bellemare, M.~G., and Bowling, M. (2017).
\newblock A {L}aplacian framework for option discovery in reinforcement
  learning.
\newblock In {\em Proceedings of the International Conference on Machine
  Learning}.

\bibitem[Machado et~al., 2018]{machado18eigenoption}
Machado, M.~C., Rosenbaum, C., Guo, X., Liu, M., Tesauro, G., and Campbell, M.
  (2018).
\newblock Eigenoption discovery through the deep successor representation.
\newblock In {\em Proceedings of the International Conference on Learning
  Representations}.

\bibitem[Mahadevan and Liu, 2010]{mahadevan10basis}
Mahadevan, S. and Liu, B. (2010).
\newblock Basis construction from power series expansions of value functions.
\newblock In {\em Advances in Neural Information Processing Systems}.

\bibitem[Mahadevan and Maggioni, 2007]{mahadevan07proto}
Mahadevan, S. and Maggioni, M. (2007).
\newblock {Proto-value functions: A Laplacian framework for learning
  representation and control in Markov decision processes}.
\newblock {\em Journal of Machine Learning Research}.

\bibitem[Menache et~al., 2005]{menache05basis}
Menache, I., Mannor, S., and Shimkin, N. (2005).
\newblock {Basis function adaptation in temporal difference reinforcement
  learning}.
\newblock {\em Annals of Operations Research}.

\bibitem[Mnih et~al., 2015]{mnih15human}
Mnih, V., Kavukcuoglu, K., Silver, D., Rusu, A.~A., Veness, J., Bellemare,
  M.~G., Graves, A., Riedmiller, M., Fidjeland, A.~K., Ostrovski, G., et~al.
  (2015).
\newblock Human-level control through deep reinforcement learning.
\newblock {\em Nature}, 518(7540):529--533.

\bibitem[Munos, 2003]{munos03error}
Munos, R. (2003).
\newblock Error bounds for approximate policy iteration.
\newblock In {\em Proceedings of the International Conference on Machine
  Learning}.

\bibitem[Munos, 2007]{munos07performance}
Munos, R. (2007).
\newblock Performance bounds in l\_p-norm for approximate value iteration.
\newblock {\em SIAM Journal on Control and Optimization}.

\bibitem[Nachum et~al., 2019]{nachum19nearoptimal}
Nachum, O., Gu, S., Lee, H., and Levine, S. (2019).
\newblock Near-optimal representation learning for hierarchical reinforcement
  learning.
\newblock In {\em Proceedings of the International Conference on Learning
  Representations}.

\bibitem[Parr et~al., 2008]{parr08analysis}
Parr, R., Li, L., Taylor, G., Painter-Wakefield, C., and Littman, M.~L. (2008).
\newblock An analysis of linear models, linear value-function approximation,
  and feature selection for reinforcement learning.
\newblock In {\em Proceedings of the International Conference on Machine
  Learning}.

\bibitem[Parr et~al., 2007]{parr07analyzing}
Parr, R., Painter-Wakefield, C., Li, L., and Littman, M. (2007).
\newblock Analyzing feature generation for value-function approximation.
\newblock In {\em Proceedings of the International Conference on Machine
  Learning}.

\bibitem[Petrik, 2007]{petrik07analysis}
Petrik, M. (2007).
\newblock An analysis of {L}aplacian methods for value function approximation
  in {MDP}s.
\newblock In {\em Proceedings of the International Joint Conference on
  Artificial Intelligence}.

\bibitem[Petrik and Zilberstein, 2011]{petrik11robust}
Petrik, M. and Zilberstein, S. (2011).
\newblock Robust approximate bilinear programming for value function
  approximation.
\newblock {\em Journal of Machine Learning Research}.

\bibitem[Puterman, 1994]{puterman94markov}
Puterman, M.~L. (1994).
\newblock {\em {M}arkov {D}ecision {P}rocesses: Discrete stochastic dynamic
  programming}.
\newblock John Wiley \& Sons, Inc.

\bibitem[Ratitch and Precup, 2004]{ratitch04sparse}
Ratitch, B. and Precup, D. (2004).
\newblock Sparse distributed memories for on-line value-based reinforcement
  learning.
\newblock In {\em Proceedings of the European Conference on Machine Learning}.

\bibitem[Rockafellar and Wets, 2009]{rockafellar09variational}
Rockafellar, R.~T. and Wets, R. J.-B. (2009).
\newblock {\em Variational analysis}.
\newblock Springer Science \& Business Media.

\bibitem[Ruan et~al., 2015]{ruan15representation}
Ruan, S.~S., Comanici, G., Panangaden, P., and Precup, D. (2015).
\newblock Representation discovery for mdps using bisimulation metrics.
\newblock In {\em Proceedings of the AAAI Conference on Artificial
  Intelligence}.

\bibitem[Rummery and Niranjan, 1994]{rummery94online}
Rummery, G.~A. and Niranjan, M. (1994).
\newblock On-line {Q}-learning using connectionist systems.
\newblock Technical report, Cambridge University Engineering Department.

\bibitem[Samuel, 1959]{samuel59some}
Samuel, A.~L. (1959).
\newblock Some studies in machine learning using the game of checkers.
\newblock {\em IBM Journal of Research and Development}.

\bibitem[Schaul et~al., 2015]{schaul15universal}
Schaul, T., Horgan, D., Gregor, K., and Silver, D. (2015).
\newblock Universal value function approximators.
\newblock In {\em Proceedings of the International Conference on Machine
  Learning}.

\bibitem[Selfridge, 1959]{selfridge59pandemonium}
Selfridge, O. (1959).
\newblock Pandemonium: {A} paradigm for learning.
\newblock In {\em Symposium on the mechanization of thought processes}.

\bibitem[Silver et~al., 2016]{silver16mastering}
Silver, D., Huang, A., Maddison, C.~J., Guez, A., Sifre, L., van~den Driessche,
  G., Schrittwieser, J., Antonoglou, I., Panneershelvam, V., Lanctot, M.,
  Dieleman, S., Grewe, D., Nham, J., Kalchbrenner, N., Sutskever, I.,
  Lillicrap, T., Leach, M., Kavukcuoglu, K., Graepel, T., and Hassabis, D.
  (2016).
\newblock Mastering the game of {G}o with deep neural networks and tree search.
\newblock {\em Nature}, 529(7587):484--489.

\bibitem[Solway et~al., 2014]{solway14optimal}
Solway, A., Diuk, C., C{\'o}rdova, N., Yee, D., Barto, A.~G., Niv, Y., and
  Botvinick, M.~M. (2014).
\newblock Optimal behavioral hierarchy.
\newblock {\em PLOS Computational Biology}.

\bibitem[Song et~al., 2016]{song16linear}
Song, Z., Parr, R., Liao, X., and Carin, L. (2016).
\newblock Linear feature encoding for reinforcement learning.
\newblock In {\em Advances in Neural Information Processing Systems}.

\bibitem[Stachenfeld et~al., 2014]{stachenfeld14design}
Stachenfeld, K.~L., Botvinick, M., and Gershman, S.~J. (2014).
\newblock Design principles of the hippocampal cognitive map.
\newblock In {\em Advances in Neural Information Processing Systems}.

\bibitem[Such et~al., 2019]{such19atari}
Such, F.~P., Madhavan, V., Liu, R., Wang, R., Castro, P.~S., Li, Y., Schubert,
  L., Bellemare, M.~G., Clune, J., and Lehman, J. (2019).
\newblock An {A}tari model zoo for analyzing, visualizing, and comparing deep
  reinforcement learning agents.
\newblock In {\em Proceedings of the International Joint Conference on
  Artificial Intelligence}.

\bibitem[Sutton et~al., 2011]{sutton11horde}
Sutton, R., Modayil, J., Delp, M., Degris, T., Pilarski, P., White, A., and
  Precup, D. (2011).
\newblock Horde: A scalable real-time architecture for learning knowledge from
  unsupervised sensorimotor interaction.
\newblock In {\em Proceedings of the International Conference on Autonomous
  Agents and Multiagents Systems}.

\bibitem[Sutton, 1995]{sutton95td}
Sutton, R.~S. (1995).
\newblock {TD} models: {M}odeling the world at a mixture of time scales.
\newblock In {\em Proceedings of the International Conference on Machine
  Learning}.

\bibitem[Sutton, 1996]{sutton96generalization}
Sutton, R.~S. (1996).
\newblock Generalization in reinforcement learning: Successful examples using
  sparse coarse coding.
\newblock In {\em Advances in Neural Information Processing Systems}.

\bibitem[Sutton and Barto, 1998]{sutton98reinforcement}
Sutton, R.~S. and Barto, A.~G. (1998).
\newblock {\em Reinforcement learning: An introduction}.
\newblock MIT Press.

\bibitem[Sutton et~al., 2016]{sutton16emphatic}
Sutton, R.~S., Mahmood, A.~R., and White, M. (2016).
\newblock An emphatic approach to the problem of off-policy temporal-difference
  learning.
\newblock {\em Journal of Machine Learning Research}.

\bibitem[Sutton et~al., 2000]{sutton00policy}
Sutton, R.~S., McAllester, D.~A., Singh, S.~P., and Mansour, Y. (2000).
\newblock Policy gradient methods for reinforcement learning with function
  approximation.
\newblock In {\em Advances in Neural Information Processing Systems}.

\bibitem[Sutton et~al., 1999]{sutton99between}
Sutton, R.~S., Precup, D., and Singh, S.~P. (1999).
\newblock Between {MDPs} and semi-{MDPs}: A framework for temporal abstraction
  in reinforcement learning.
\newblock {\em Artificial Intelligence}.

\bibitem[Tesauro, 1995]{tesauro95temporal}
Tesauro, G. (1995).
\newblock Temporal difference learning and {TD}-{G}ammon.
\newblock {\em Communications of the ACM}, 38(3).

\bibitem[Tieleman and Hinton, 2012]{tieleman12rmsprop}
Tieleman, T. and Hinton, G. (2012).
\newblock Rms{P}rop: {D}ivide the gradient by a running average of its recent
  magnitude.
\newblock COURSERA: Neural Networks for Machine Learning.

\bibitem[Tosatto et~al., 2017]{tosatto17boosted}
Tosatto, S., Pirotta, M., D'Eramo, C., and Restelli, M. (2017).
\newblock Boosted fitted q-iteration.
\newblock In {\em Proceedings of the International Conference on Machine
  Learning}.

\bibitem[van~den Oord et~al., 2018]{vandenoord18representation}
van~den Oord, A., Li, Y., and Vinyals, O. (2018).
\newblock Representation learning with contrastive predictive coding.
\newblock In {\em Advances in Neural Information Processing Systems}.

\bibitem[van Hasselt et~al., 2018]{vanhasselt18deep}
van Hasselt, H., Doron, Y., Strub, F., Hessel, M., Sonnerat, N., and Modayil,
  J. (2018).
\newblock Deep reinforcement learning and the deadly triad.
\newblock {\em arXiv}.

\bibitem[Wu et~al., 2019]{wu19laplacian}
Wu, Y., Tucker, G., and Nachum, O. (2019).
\newblock The laplacian in rl: Learning representations with efficient
  approximations.
\newblock In {\em Proceedings of the International Conference on Learning
  Representations (to appear)}.

\bibitem[Yu and Bertsekas, 2009]{yu09basis}
Yu, H. and Bertsekas, D.~P. (2009).
\newblock Basis function adaptation methods for cost approximation in mdp.
\newblock In {\em Proceedings of the {IEEE} Symposium on Adaptive Dynamic
  Programming and Reinforcement Learning}.

\end{thebibliography}

\newpage

\appendix

\section{Proof of Lemma \ref{lem:extremal_vertices}}

\def \bellopsig {\mathcal{T}_{\sigma}}
\def \bellopsigi {\mathcal{T}_{\sigma_i}}

Consider the value polytope $\cV$. We have using Corollary 1 of \citet{dadashi19value} that 
\begin{align}\cV \subseteq \Conv(\cV) = \Conv(V^{\pi_1},\dots,V^{\pi_m}),
\label{eqn:convex_hull}
\end{align}
where $\pi_1,\dots, \pi_m$ is a finite collection of deterministic policies. We assume that this set of policies is of minimal cardinality e.g. the value functions $V^{\pi_1},\dots,V^{\pi_m}$ are distinct.\\

The optimization problem $\max\limits_{V \in \cV} \delta^\top V$ is equivalent to the linear program $\max\limits_{V \in \Conv(\cV)} \delta^\top V$, and
the maximum is reached at a vertex $U$ of the convex hull of $\cV$ \citep{boyd04convex}. By \eqnref{convex_hull}, $U$ is the value function of a deterministic policy. Now consider $\delta \in \bR^n$ such that $f_\delta$ attains its maximum over multiple elements of the convex hull. By hypothesis, there must be two policies $\pi_i$, $\pi_j$ such that $V^{\pi_i} \ne V^{\pi_j}$ and
\begin{equation*}
\max\limits_{V \in \cV} \delta^\top V = \delta^\top V^{\pi_i} = \delta^\top V^{\pi_j},
\end{equation*}
and thus
\begin{equation}
\label{eqn:multiple_max_vertices}
\delta^\top (V^{\pi_i} - V^{\pi_j}) = 0 .
\end{equation}
Write $\Delta$ for the ensemble of such $\delta$. We have from \eqnref{multiple_max_vertices}: 
$$\Delta \subseteq \bigcup\limits_{1\leq i < j\leq m} \{\delta \in \bR^n \, | \, \delta^T (V_{\pi_i} - V_{\pi_j})=0\}.$$
As $V^{\pi_1},\dots,V^{\pi_m}$ are distinct, $\Delta$ is included in a finite union of hyperplanes (recall that hyperplanes of $\bR^n$ are vector spaces of dimension $n-1$). The Lebesgue measure of a hyperplane is 0 (in $\bR^n$), hence a finite union of hyperplanes also has Lebesgue measure 0. Hence $\Delta$ itself has Lebesgue measure of 0 in $\bR^n$.

\section{Proof of Corollary \ref{cor:2n_distinct_avfs}}\label{sec:at_most_2n_avfs}

Similarly to the proof of Lemma 1, we introduce $V^{\pi_1},\dots,V^{\pi_m}$ which are the distinct vertices of the convex hull of the value polytope $\cV$. Note that $\pi_1, \dots, \pi_m$ are deterministic policies. We shall show that there are at most $2^n$ such vertices. 

Recall the definition of a cone in $\bR^n$: $C$ is a cone in $\bR^n$ if $\forall v \in C, \forall \alpha \ge 0, \alpha v \in C$. For each vertex $V^{\pi_i}$, \citet[][Theorem 6.12]{rockafellar09variational} states that there is an associated cone $C_{i}$ of nonzero Lebesgue measure in $\bR^n$ such that 
$$\forall \delta \in C_{i}, \, \argmax_{V \in \cV} \delta^\top V = V^{\pi_i}.$$ 
Now using Theorem \ref{thm:dual_delta_method}, we have
$$\max_{V \in \cV} \delta^\top V = \max_{\pi \in \SetPi} d^\top_\pi r, 
\text{ where } d_\pi = (I - \gamma {P^\pi}^\top)^{-1} \delta.$$
For all $\delta \in C_i$ the corresponding policy $\pi_i$ is the same (by hypothesis). For such a $\delta$, define $d_{\pi_i, \delta} := (I - \gamma {P^{\pi_i}}^\top)^{-1} \delta$, such that
$$\delta^\top V^{\pi_i} = d_{\pi_i, \delta}^\top r.$$
Because $C_i$ is a cone of nonzero Lebesgue measure in $\bR^n$, we have $span(C_i) = \bR^n$. 
Combined with the fact that $(I - \gamma {P^{\pi_i}}^\top)^{-1}$ is full rank, this implies we can find a direction $\delta_i$ in $C_i$ for which $d_{\pi_i, \delta_i}(x) \neq 0$ for all $x \in \cX$.
For this $\delta_i$, using Theorem \ref{thm:dual_delta_method} we have:
\begin{equation}
V^{\pi_i}(x) = r(x) + \gamma \left \{ \begin{array}{ll} 
  \max_{a \in \cA} \expect_{x' \sim P} V^{\pi_i}(x') & d_{\pi_i, \delta_i}(x) > 0, \\
  \min_{a \in \cA} \expect_{x' \sim P} V^{\pi_i}(x') & d_{\pi_i, \delta_i}(x) < 0, \\
  \end{array} \right .
\label{eq:max_V_from_d_pi}
\end{equation}
and each state is ``strictly'' a maximizer or minimizer (the purpose of our cone argument was to avoid the undefined case where $d_{\pi_i, \delta_i}(x) = 0$). Now define $\sigma_i \in \{-1, 1\}^n$, $\sigma_i(x) = \text{sign}(d_{\pi_i, \delta_i}(x))$. We have:
\begin{align*}
V^{\pi_i}(x) &= r(x) + \gamma \sigma_i(x) \max_{a \in \cA} \sigma_i(x) \expect_{x' \sim P} V^{\pi_i}(x') \\
&= \bellopsigi V^{\pi_i}(x)
\end{align*}
where $\bellopsig V(x) = r(x) + \gamma \sigma(x) \max_{a \in \cA} \sigma(x) \expect_{x' \sim P} V(x')$ for $\sigma \in \{-1, 1\}^n$. We show that $\bellopsig$ is a contraction mapping: for any $x \in \cX$ and $\sigma \in \{-1, 1\}^n$,
\begin{align*}
| \bellopsig V_1(x) - \bellopsig V_2(x) | &= |r(x) + \gamma \sigma(x) \max_{a \in \cA} \sigma(x) \expect_{x' \sim P} V_1(x') 
- r(x) - \gamma \sigma(x) \max_{a \in \cA} \sigma(x) \expect_{x' \sim P} V_2(x') | \\
 &= \gamma | \max_{a \in \cA} \sigma(x) \expect_{x' \sim P} V_1(x') 
- \max_{a \in \cA} \sigma(x) \expect_{x' \sim P} V_2(x') | \\
& \leq \gamma \max_{a \in \cA} | \sigma(x) \expect_{x' \sim P} V_1(x') 
- \sigma(x) \expect_{x' \sim P} V_2(x') | \\
& \leq \gamma \max_{a \in \cA} \max_{x' \in \cX} | V_1(x') - V_2(x') | \\
& = \gamma \max_{x' \in \cX} | V_1(x') - V_2(x') |.
\end{align*}

Therefore, $\|\bellopsig V_1 - \bellopsig V_2\|_{\infty} \leq \gamma \|V_1 - V_2\|_{\infty}$ and $\bellopsig$ is a $\gamma$-contraction in the supremum norm. By Banach's fixed point theorem $V^{\pi_i}$ is its a unique fixed point.

We showed that each vertex $V^{\pi_i}$ of the value function polytope $\cV$ is the fixed point of an operator $\bellopsigi$. Since there are $2^n$ such operators, there are at most $2^n$ vertices.

\def \Vpimax {{V^\pi_{\textsc{max}}}}

\section{Proof of Theorem~\ref{thm:equivalent_optimization_problems}}

We will show that the maximization over $\SetPi$ is the same as the maximization over $\SetPi_v$.

Let $\Piphi$ be the projection matrix onto the hyperplane $H$ spanned by the basis induced by $\phi$. We write
\begin{align*}
\norm{\hVphipi - \Vpi}^2_2 &= \norm{ \Piphi \Vpi - \Vpi}^2_2 \\
&= \norm{(\Piphi - I) \Vpi}^2_2 \\
&= \Vpi^\top (\Piphi - I)^\top (\Piphi - I) \Vpi \\
&= \Vpi^\top (I - \Piphi) \Vpi
\end{align*}
because $\Piphi$ is idempotent. The eigenvalues of $A = I - \Piphi$ are 1 and 0, and the eigenvectors corresponding to eigenvalue 1 are normal to $H$. Because we are otherwise free to choose any basis spanning the subspace normal to $H$, there is a unit vector $\delta$ normal to $H$ for which 
\begin{align*}
\max_\pi \norm{\hVphipi - \Vpi}^2_2 &= \max_{V^\pi \in \cV} \norm{\hVphipi - \Vpi}^2_2 \\
&= \max_{V^\pi \in \cV} \Vpi^\top \delta \delta^\top \Vpi . 
\end{align*}
Denote the value function maximizing this quantity by $\Vpimax$. This $\delta$ can be chosen so that $\delta^\top \Vpimax > 0$ (if not, take $\delta' = -\delta$). Then $\Vpimax$ is also the maximizer of $f(V) := \delta^\top V$ over $\cV$, and Lemma \ref{lem:extremal_vertices} tells us that $\Vpimax$ is an extremal vertex.

\section{Proof of Theorem~\ref{thm:dual_delta_method}}

\def \cpi {C^\pi}
\def \Dpi {D^\pi}
\def \ppit {{Ppi}^\top}
\def \api {A^\pi}
\def \bpi {B^\pi}
\def \apit {{A^\pi}^\top}
\def \bpit {{B^\pi}^\top}
\def \cpit {{C^\pi}^\top}
\def \Dpit {{D^\pi}^\top}
\def \xs {x^*}
\def \ar {(I - \gamma \api)}
\def \art {(I - \gamma {A^\pi}^\top)}

To begin, note that
\begin{align*}
\delta^\top \Vpi &= \delta^\top (I - \gamma \Ppi)^{-1} r \\
&= (I - \gamma \Ppi^\top)^{-1} \delta^\top r \\
&= \dpi^\top r,
\end{align*}
as required.

Now, we choose an indexing for states in $\cS$ and will refer to states by their index.

Let $\pi$ be the policy maximizing $\delta^\top \Vpi$ and consider some $\xs \in \cS$. We assume without loss of generality that $\xs$ is the first state in the previous ordering. Recall that $n = |\cS|$.

The theorem states that policy $\pi$ chooses the highest-valued action at $\xs$ if $\dpi(\xs) > 0$, and the lowest-valued action if $\dpi(\xs) < 0$. Writing $P^\pi_{\xs} := P^\pi(\cdot \cbar \xs)$ for conciseness, this is equivalent to
\begin{equation*}
  r(\xs) + \expect_{x' \sim P^\pi_{\xs}} V^\pi(x') = \max_{\pi'} r(\xs) + \expect_{x' \sim P^{\pi'}_{\xs}} V^\pi(x'), 
\end{equation*}
for $\dpi(\xs) > 0$, and conversely with a $\min_{\pi'}$ for $\dpi(\xs) < 0$ (equivalently, $\cT^\pi V^\pi(\xs) \ge \cT^{\pi'} V^\pi(\xs)$ for all $\pi' \in \SetPi$ or $\cT^\pi V^\pi(\xs) \le \cT^{\pi'} V^\pi (\xs)$ in operator notation).

We write the transition matrix $\Ppi$ as follows
\begin{equation*}
    \Ppi = \begin{pmatrix}
           L^\pi_1 \\
           \vdots \\
           L^\pi_{n}
         \end{pmatrix}.
\end{equation*}
Where $L^\pi_i =\big( P^\pi(x_1 \cbar x_i), \cdots, P^\pi(x_{n} \cbar x_i) \big)$ is $\Ppi$'s $i$-th row.

Then we express the transition matrix as $\Ppi = \api + \bpi$, with $\api$ and $\bpi$ given by
\begin{equation*}
    \api = \begin{pmatrix}
           0 \\
           L^\pi_2 \\
           \vdots \\
           L^\pi_{n}
         \end{pmatrix} \quad \bpi = \begin{pmatrix}
           L^\pi_1 \\
           0 \\
           \vdots \\
           0
         \end{pmatrix}.
\end{equation*}

We can then write
\begin{align*}
    \Vpi &= r + \gamma \Ppi \Vpi \\
         &= r + \gamma (\api + \bpi) \Vpi  \\
   \Rightarrow \; \Vpi &= \ar^{-1} (r + \gamma \bpi \Vpi) .
\end{align*}
This is an application of matrix splitting \citep[e.g][]{puterman94markov}. The invertibility of $\ar$ is guaranteed because $\api$ is a substochastic matrix. 
The first term of the r.h.s corresponds to the expected sum of discounted rewards when following $\pi$ until reaching $\xs$, while the second term is the expected sum of discounted rewards received after leaving from $\xs$ and following policy $\pi$. 

Note that $(I - \gamma \api)^{-1}$ does not depend on $\pi(\cdot \cbar \xs)$ and that 
\begin{equation*}
    \bpi \Vpi =   \begin{pmatrix}
           \expect_{x' \sim P^\pi_{\xs}}  \Vpi (x') \\
           0 \\
           \vdots \\
          0
         \end{pmatrix}.
\end{equation*}
Write $\cpi = (I - \gamma \apit)^{-1} \delta$. We have
\begin{align*}
    \delta^\top \Vpi &= \delta^\top \ar^{-1} (r + \gamma \bpi \Vpi) \\
                  &= {\cpi}^\top (r + \gamma \bpi \Vpi) \\
                  &= {\cpi}^\top r + \cpi(\xs) \expect_{x' \sim P^\pi_{\xs}}  \Vpi (x').
\end{align*}
Now by assumption,
\begin{equation}\label{eqn:greater_than_vpi}
\delta^\top V^\pi \ge \delta^\top V^{\pi'}
\end{equation}
for any other policy $\pi' \in \SetPi$. Take $\pi'$ such that $\pi'(\cdot \cbar x) = \pi(\cdot \cbar x)$ everywhere but $\xs$; then $\cpi = C^{\pi'}$ and \eqnref{greater_than_vpi} implies that
\begin{equation*}
C^\pi(\xs) \expect_{x' \sim P^\pi_{\xs}} \Vpi(x') \ge C^\pi(\xs) \expect_{x' \sim P^{\pi'}_{\xs}} V^\pi(x') .
\end{equation*}
Hence $\pi$ must pick the maximizing action in $\xs$ if $C^\pi(\xs) > 0$, and the minimizing action if $C^\pi(\xs) < 0$. 

To conclude the proof, we show that $\dpi (\xs)$ and $\cpi (\xs)$ have the same sign. We write 
\begin{align*}
    \dpi &= \delta + \gamma (\apit + \bpit) \dpi.
\end{align*}
Then
\begin{align*}
\art \dpi &= \delta + \gamma \bpit \dpi \\
\Rightarrow \quad \dpi &= \cpi + \gamma \art^{-1} \bpit \dpi \\
         &= \sum_{k=0}^{\infty} (\gamma \art^{-1} \bpit)^k \cpi \\
         &= \sum_{k=0}^{\infty} \gamma^k ({\Dpi}^\top)^k \cpi. \\
\end{align*}
Where $\Dpi = \bpi \ar^{-1}$ is a matrix with non-negative components. Because $\bpi$ is sparse every row of $(\Dpi)^k$ is null except for the first one. We can write
\begin{equation*}
    ({\Dpi}^k)^\top = \begin{pmatrix}
           d_{11}^k \, 0 \cdots 0 \\
           \vdots  \\
           d_{1n}^k \, 0 \cdots 0 
         \end{pmatrix} \quad \forall i, \, d_{1i}^k \geq 0.
\end{equation*}
And
\begin{equation*}
   \dpi (\xs) =  \big( \sum_{k=0}^{\infty} \gamma^k d_{11}^k \big) \, \cpi (\xs).
\end{equation*}
Hence $\cpi (\xs)$ and $\dpi (\xs)$ have the same sign.

\section{Proof of Theorem~\ref{thm:optimal_representation_is_svd}}

We first transform \eqnref{relaxed_learning_problem} in a equivalent problem.
Let $V \in \bR^n$, and denote by $\hat V_\phi := \Piphi V$ the orthogonal projection of $V$ onto the subspace spanned by the columns of $\Phi$.
From Pythagoras' theorem we have, for any $V \in \bR^n$
\begin{align*}
\norm{V}_2^2 = \norm{\hat V_\phi - V}_2^2 + \norm{\hat V_\phi}_2^2 
\end{align*}
Then
\begin{align*}
\min_{\phi \in \SetPhi} \expect_{V \sim \xi} \norm{\hat V_\phi - V}_2^2 
&= \min_{\phi \in \SetPhi} \expect_{V \sim \xi} \big[ \norm{V}^2_2 - \norm{\Piphi V}^2_2 \big] \\
&= \max_{\phi \in \SetPhi} \expect_{V \sim \xi} \norm{\Piphi V}^2_2.
\end{align*}

Let $u_1^*, \dots, u_d^*$ the principal components defined in Theorem~\ref{thm:optimal_representation_is_svd}. These form an orthonormal basis. Hence $u_1^*, \dots, u_d^*$ is equivalently a solution of
\begin{align*}
\max_{\substack{u_1, \dots, u_d \in \bR^n \\ \text{orthonormal}}} \sum_{i=1}^{d} \expect_{V \sim \xi} (u_i^\top V)^2_2
&= \max_{\substack{u_1, \dots, u_d \in \bR^n \\ \text{orthonormal}}} \sum_{i=1}^{d} \expect_{V \sim \xi} \norm{u_i^\top V u_i}^2_2 \\
&= \max_{\substack{u_1, \dots, u_d \in \bR^n \\ \text{orthonormal}}} \expect_{V \sim \xi} \norm{ \sum_{i=1}^{d} u_i^\top V u_i}^2_2 \\
&= \max_{\substack{\Phi = [u_1, \dots, u_d] \\ \Phi^\top \Phi = I}} \expect_{V \sim \xi} \norm{ \Phi \Phi^T V}^2 \\
&= \max_{\phi \in \SetPhi} \expect_{V \sim \xi} \norm{\Piphi V}^2_2.
\end{align*}
Which gives the desired result.
The equivalence with the eigenvectors of $\expect_\xi V V^\top$ follows from writing 
\begin{align*}
\expect_{V \sim \xi} (u^\top V)^2_2 &= \expect_{V \sim \xi} u^\top V V^\top u \\
&= u^\top \expect_{V \sim \xi} \big [ V V^\top \big ] u 
\end{align*}
and appealing to a Rayleigh quotient argument, since we require $u^*_i$ to be of unit norm.

\section{The Optimization Problem \eqnref{representation_learning_problem} as a Quadratic Program}

\begin{prop}\label{prop:optimization}
The optimization problem \eqnref{representation_learning_problem} is equivalent to a quadratic program with quadratic constraints. 
\end{prop}
\begin{proof}
For completeness, let $n$, $d$ be the number of states and features, respectively. We consider representations $\Phi \in \bR^{n \times d}$. Recall that $\Piphi$ is the projection operator onto the subspace spanned by $\Phi$, that is
\begin{equation*}
\Piphi = \Phi \big (\Phi^\top \Phi\big)^{-1} \Phi^\top.
\end{equation*}
We will also write $\SetPi_d$ for the space of deterministic policies. We write \eqnref{representation_learning_problem} in epigraph form \citep{boyd04convex}:
\begin{align*}
\text{min. } \max_{\pi} \norm{\Piphi \Vpi - \Vpi}^2_2 \Leftrightarrow \\
\text{min. } \max_{\pi \in \SetPi_d} \norm{\Piphi \Vpi - \Vpi}^2_2 \Leftrightarrow \\
\text{min. } t \quad  \text{ s.t.} \norm{\Piphi \Vpi - \Vpi}^2_2 \le t \; \forall \pi \in \SetPi_d .
\end{align*}
The first equivalence comes from the fact that the extremal vertices of our polytope are achieved by deterministic policies.
The norm in the constraint can be written as
\begin{align*}
\norm{\Piphi \Vpi - \Vpi}^2_2 &= \norm{(\Piphi - I) \Vpi}^2_2 \\
&= \Vpi^\top (\Piphi - I)^\top (\Piphi - I) \Vpi \\
&= \Vpi^\top (\Piphi - I)^\top (\Piphi - I) \Vpi \\
&\overset{(a)}{=} \Vpi^\top (\Piphi^2 - 2 \Piphi + I) \Vpi \\
&\overset{(b)}{=} \Vpi^\top (I - \Piphi) \Vpi,
\end{align*}
where $(a)$ and $(b)$ follow from the idempotency of $\Piphi$. This is 
\begin{equation*}
\norm{\Piphi \Vpi - \Vpi}^2_2 = \Vpi^\top \big (I - \Phi (\Phi^\top \Phi)^{-1} \Phi^\top \big ) \Vpi .
\end{equation*}
To make the constraint quadratic, we further require that the representation be left-orthogonal: $\Phi^\top \Phi = I$. Hence the optimization problem \eqnref{representation_learning_problem} is equivalent to
\begin{equation*}
\text{minimize } \; t \quad \text{ s.t. }
\end{equation*}
\begin{equation*}
\qquad \Vpi^\top (I - \Phi \Phi^\top) \Vpi \le t \quad \forall \pi \in \SetPi_d
\end{equation*}
\begin{equation*}
\qquad \Phi^\top \Phi = I .
\end{equation*}
From inspection, these constraints are quadratic.
\end{proof}
However, there are an exponential number of deterministic policies and hence, an exponential number of constraints in our optimization problem. 

\section{NP-hardness of Finding AVFs}

\begin{prop}\label{prop:np-hard}
Finding $\max_{\pi \in \SetPi_d} \delta^\top V^\pi$ is NP-hard,
where the input is a deterministic MDP with binary-valued reward function, discount rate $\gamma = 1/2$ and $\delta : \cX \to \{-1/4,0,1\}$.
\end{prop}

We use a reduction from the optimization version of minimum set cover, which is known to be NP-hard \citep[Corollary 15.24]{BV08}.
Let $n$ and $m$ be natural numbers. An instance of set cover 
is a collection of sets $\cC = \{C_1,\ldots,C_m\}$ where $C_i \subseteq [n] = \{1,2,\ldots,n\}$ for all $i \in [m]$. 
The minimum set cover problem is
\begin{align*}
\min_{\cJ \subseteq [m]} \left\{|\cJ| : \bigcup_{j \in \cJ} C_j = [n]\right\}\,. 
\end{align*}

Given a Markov decision process $\langle \cX, \cA, r, P, \gamma \rangle$ and function
$\delta : \cX \to [-1,1]$ define
\begin{align*}
R(\pi) = \sum_{x \in \cX} \delta(x) V^\pi(x)\,.
\end{align*}
We are interested in the optimization problem
\begin{align}
\max_{\pi \in \SetPi_d} R(\pi)\,.
\label{eq:hard-opt}
\end{align}
When $\delta(x) \geq 0$ for all $x$ this corresponds to finding the usual optimal policy, which can be found efficiently
using dynamic programming.
The propositions claims that more generally the problem is NP-hard. 

Consider an instance of set cover $\cC = \{C_1,\ldots,C_m\}$ over universe $[n]$ with $m > 1$.
Define a deterministic MDP $\langle \cX, \cA, r, P, \gamma \rangle$ with $\gamma = 1/2$ and $n + m + 2$ states and at most $m$ actions.
The state space is $\cX = \cX_1 \cup \cX_2 \cup \cX_3$ where
\begin{align*}
\cX_1 &= \{u_1,\ldots,u_n\} & \cX_2 &= \{v_1,\ldots,v_m\} & \cX_3 &= \{g, b\}\,.
\end{align*}
The reward function is $r(x) = \indic{x = g}$.
The transition function in a deterministic MDP is characterized by a function mapping states to the set of possible next states: 
\begin{align*}
N(x) = \bigcup_{a \in \cA} \{x' : P(x' \,|\, x,a) = 1\}\,. 
\end{align*}
We use $\cC$ to choose $P$ as a deterministic transition function for which 
\begin{align*}
N(x) = 
\begin{cases}
\{x\} & \text{if } x \in \cX_3 \\
\{g, b\} & \text{if } x \in \cX_2 \\
\{v_j : i \in C_j\} & \text{if } x = u_i \in \cX_1\,.
\end{cases}
\end{align*}
This means the states in $\cX_3$ are self transitioning and states in $\cX_2$ have transitions leading to either state in $\cX_3$.
States in $\cX_1$ transition to states in $\cX_2$ in a way that depends on the set cover instance.
The situation is illustrated in Figure~\ref{fig:set-cover}.
Since both policies and the MDP are deterministic, we can represent a policy as a function $\pi : \cX \to \cX$ for which $\pi(x) \in N(x)$ for all $x \in \cX$.
To see the connection to set cover, notice that 
\begin{align}
\bigcup_{v_j \in \pi(\cX_1)} C_j = [n]\,,
\label{eq:set-cover-connection}
\end{align}
where $\pi(\cX_1) = \{\pi(x) : x \in \cX_1\}$. 
Define 
\begin{align*}
\delta(x) = \begin{cases}
1 & \text{if } x \in \cX_1 \\
-1/4 & \text{if } x \in \cX_2 \\
0 & \text{if } x \in \cX_3\,.
\end{cases}
\end{align*}
Using the definition of the value function and MDP,
\begin{align*}
R(\pi) 
&= \sum_{x \in \cX} \delta(x) V^\pi(x) \\
&= \sum_{x \in \cX_1} V^\pi(x) - \frac{1}{4} \sum_{x \in \cX_2} V^\pi(x) \\
&= \sum_{x \in \cX_1} V^\pi(x) - \frac{1}{4} \sum_{x \in \cX_2} \indic{\pi(x) = g} \\
&= \frac{1}{2} \sum_{x \in \cX_1} \indic{\pi(\pi(x)) = g} - \frac{1}{4} \sum_{x \in \cX_2} \indic{\pi(x) = g}\,. 
\end{align*}
The decomposition shows that any policy maximizing (\ref{eq:hard-opt}) must satisfy
$\pi(\pi(\cX_1)) = \{g\}$ and $\pi(\cX_2 \setminus \pi(\cX_1)) = \{b\}$ and for such policies
\begin{align*}
R(\pi) &= \frac{1}{2} \left(n - \frac{1}{2} |\pi(\cX_1)|\right)\,.
\end{align*}
In other words, a policy maximizing (\ref{eq:hard-opt}) minimizes $|\pi(\cX_1)|$, which by (\ref{eq:set-cover-connection})
corresponds to finding a minimum set cover.
Rearranging shows that
\begin{align*}
\min_{\cJ \subseteq [m]} \left\{|\cJ| : \bigcup_{j \in \cJ} C_j = [n]\right\}
= 2n - 4\max_{\pi \in \SetPi_d} R(\pi) \,.
\end{align*}
The result follows by noting this reduction is clearly polynomial time.

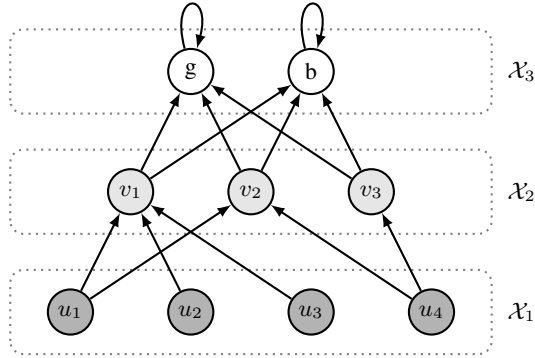
\begin{figure}[h]
\centering
\tikzset{>=latex}
\begin{tikzpicture}[thick,scale=0.8,font=\small]
\tikzstyle{n1} = [circle,fill=black!30!white,draw,inner sep=2pt,minimum width=0.6cm]
\tikzstyle{n2} = [circle,fill=black!10!white,draw,inner sep=2pt,minimum width=0.6cm]
\tikzstyle{n3} = [circle,draw,inner sep=2pt,minimum width=0.6cm]
\tikzstyle{b} = [dotted,rounded corners=5pt,line width=0.8pt,gray]

\draw[b] (-1,-0.7) rectangle (7,0.7);
\draw[b] (-1,1.3) rectangle (7,2.7);
\draw[b] (-1,3.3) rectangle (7,4.7);
\node[anchor=west] at (7.1,0) {$\cX_1$};
\node[anchor=west] at (7.1,2) {$\cX_2$};
\node[anchor=west] at (7.1,4) {$\cX_3$};

\node[n1] (1) at (0,0) {$u_1$};
\node[n1] (2) at (2,0) {$u_2$};
\node[n1] (3) at (4,0) {$u_3$};
\node[n1] (4) at (6,0) {$u_4$};
\node[n2] (5) at (1,2) {$v_1$};
\node[n2] (6) at (3,2) {$v_2$};
\node[n2] (7) at (5,2) {$v_3$};
\node[n3] (g) at (2,4) {g};
\node[n3] (b) at (4,4) {b};
\draw[-latex] (g) edge[loop above,min distance=1cm] node {} (g);
\draw[-latex] (b) edge[loop above,min distance=1cm] node {} (b);
\draw[-latex] (1) -- (5);
\draw[-latex] (1) -- (6);
\draw[-latex] (2) -- (5);
\draw[-latex] (3) -- (5);
\draw[-latex] (4) -- (7);
\draw[-latex] (4) -- (6);
\draw[-latex] (5) -- (g);
\draw[-latex] (6) -- (g);
\draw[-latex] (7) -- (g);
\draw[-latex] (5) -- (b);
\draw[-latex] (6) -- (b);
\draw[-latex] (7) -- (b);
\end{tikzpicture}
\caption{The challenging MDP given set cover problem $\{\{1,2,3\}, \{1,4\}, \{4\}\}$. 
State $g$ gives a reward of $1$ and all other states give reward $0$.
The optimal policy is to find the smallest subset of the middle layer such
that for every state in the bottom layer there exists a transition to the subset.}\label{fig:set-cover}
\end{figure}

\section{Empirical Studies: Methodology}\label{sec:empirical_studies_methodology}

\subsection{Four-room Domain}\label{sec:four_room_domain}

The four-room domain consists of 104 discrete states arranged into four ``rooms''. There are four actions available to the agent, transitioning deterministically from one square to the next; when attempting to move into a wall, the agent remains in place. In our experiments, the top right state is a goal state, yielding a reward of 1 and terminating the episode; all other transitions have 0 reward.
\subsection{Learning $\phi$}

Our representation $\phi$ consists of a single hidden layer of 512 rectified linear units (ReLUs) followed by a layer of $d$ ReLUs which form our learned features. The use of ReLUs has an interesting side effect that all features are nonnegative, but other experiments with linear transforms yielded qualitatively similar results. The input is a one-hot encoding of the state (a 104-dimensional vector). All layers (and generally speaking, experiments) also included a bias unit.

The representation was learned using standard deep reinforcement learning tools taken from the Dopamine framework \citep{castro18dopamine}. Our loss function is the mean squared loss w.r.t. the targets, i.e. the AVFs or the usual value function. The losses were then trained using RMSProp with a step size of 0.00025 (the default optimizer from Dopamine), for 200,000 training updates each over a minibatch of size 32; empirically, we found our results robust to small changes in step sizes.

In our experiments we optimize both parts of the two-part approximation defined by $\phi$ and $\theta$ simultaneously, with each prediction made as a linear combination of features $\phi(x)^\top \theta_i$ and replacing $\tilde L(\phi; \bmu)$ from \eqnref{auxiliary_tasks_loss} with a sample-based estimate. This leads to a slightly different optimization procedure but with similar representational characteristics. 

\subsection{Implementation Details: Proto-Value Functions}\label{sec:pvf_implementation}

Our \textsc{pvf} representation consists in the top $k$ left-singular vectors of the successor representation $(I - \gamma \Ppi)^{-1}$ for $\pi$ the uniformly random policy, as suggested by \citet{machado18eigenoption,behzadian18feature}. See Figure \ref{fig:proto-value-functions} for an illustration.

\subsection{Learning AVFs}

The AVFs were learned from 1000 policy gradient steps, which were in general sufficient for convergence to an almost-deterministic policy. This policy gradient scheme was defined by directly writing the matrix $(I - \gamma P^\pi)^{-1}$ as a Tensorflow op \citep{abadi16tensorflow} and minimizing $-\delta^\top(I - \gamma P^\pi)^{-1} r$ w.r.t. $\pi$. We did not use an entropy penalty. In this case, there is no approximation: the AVF policies are directly represented as matrices of parameters of softmax policies.

\subsection{SARSA}

In early experiments we found LSPI and fitted value iteration to be somewhat unstable and eventually converged on a relatively robust, model-based variant of SARSA.

In all cases, we define the following dynamics. We maintain an occupancy vector $d$ over the state space. At each time step we update this occupancy vector by applying one transition in the environment according to the current policy $\pi$, but also mix in a probability of resetting to a state uniformly at random in the environment:
\begin{equation*}
d = 0.99 d \Ppi + 0.01 \text{Unif}(\cX)
\end{equation*}
The policy itself is an $\epsilon$-greedy policy according to the current $Q$-function, with $\epsilon=0.1$. 

We update the $Q$-function using a semi-gradient update rule based on expected SARSA \citep{sutton98reinforcement}, but where we simultaneously compute updates across all states and weight them according to the occupancy $d$. We use a common step size of 0.01 but premultiplied the updates by the pseudoinverse of $\Phi^\top \Phi$ to deal with variable feature shapes across methods. This process was applied for 50,000 training steps, after which we report performance as the average value and/or number of steps to goal for the 10 last recorded policies (at intervals of 100 steps each).

Overall, we found this learning scheme to reduce experimental variance and to be robust to off-policy divergence, which we otherwise observed in a number of experiments involving value-only representations.

\section{Representations as Principal Components of Sets of Value Functions}\label{sec:representations_from_svd}

In the main text we focused on the use of value functions as auxiliary tasks, which are combined into the representation loss \eqnref{auxiliary_tasks_loss}. However, Section \ref{sec:relationship_to_auxiliary_tasks} shows that doing so is equivalent (in intent) to computing the principal components of a particular set of value functions, where each ``column'' corresponds to a particular auxiliary task.

In Figure \ref{fig:representations-from-svd} we show the representations generated from this process, using different sets of value functions. For completeness, we consider:
\begin{itemize}
	\item 1000 AVFs,
	\item 1000 random deterministic policies (RDPs),
	\item 1000 random stochastic policies (RSPs), and
	\item The 104 rows of the successor matrix (corresponding to proto-value functions).
\end{itemize}
As with principal component analysis, the per-state feature activations are determined up to a signed scalar; we pick the vector which has more positive components than negative. In all but the PVF case, we sample a subset of the many possible value functions within a particular set. Figure \ref{fig:variable_number_of_avfs} shows that the AVF approach is relatively robust to the sample size.

The AVFs are sampled using Algorithm \ref{alg:AVF_representation_learning}, i.e. by sampling a random interest function $\delta \in [-1, 1]^n$ and using policy gradient on a softmax policy to find the corresponding value function. The random policies were generated by randomly initializing the same softmax policy and using them as-is (RSPs) or multiplying the logits by 1e6 (RDPs).

\begin{figure*}[htb]
\center{
\includegraphics[width=5.7in]{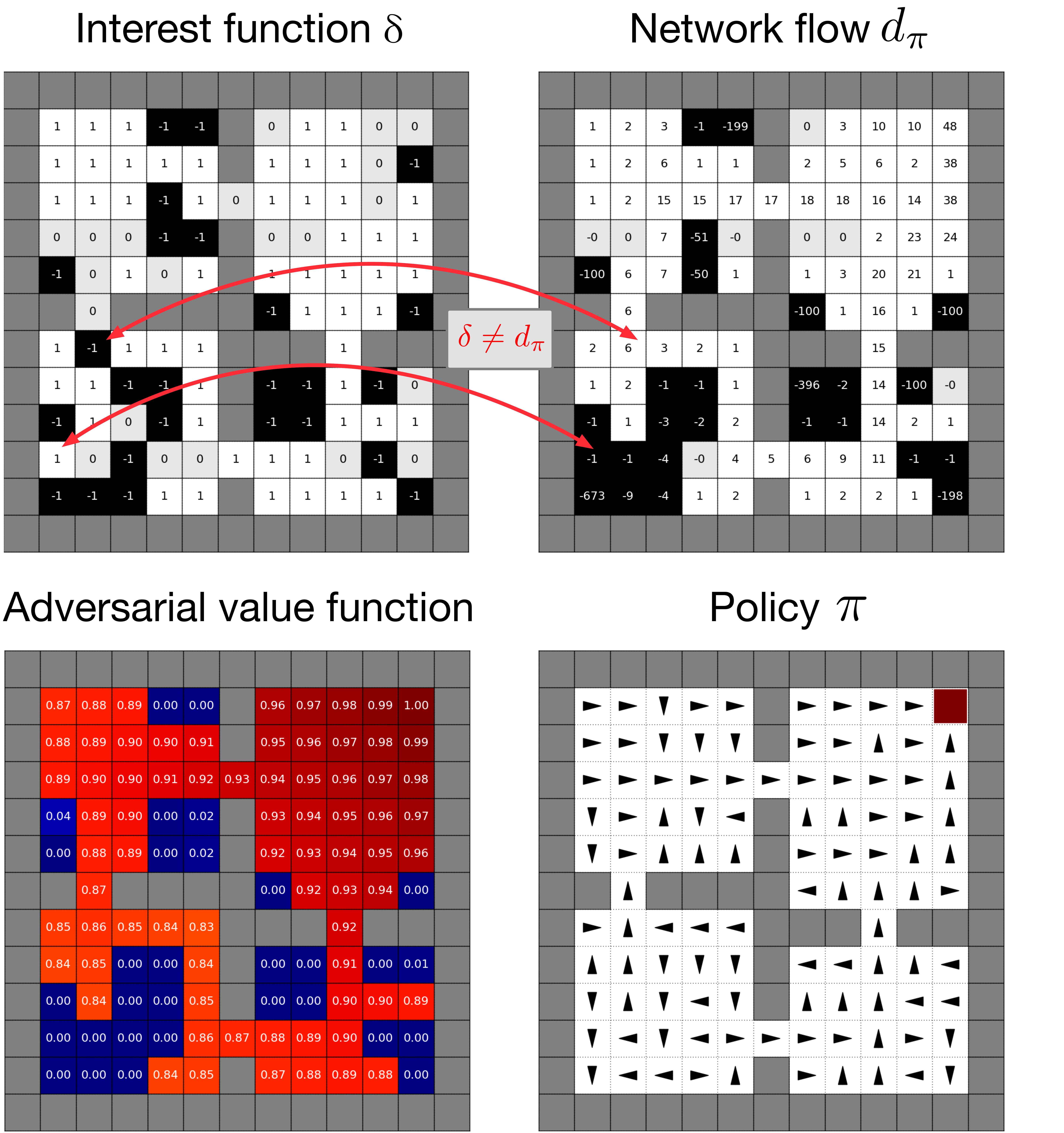}
\caption{Figure \ref{fig:alignment-and-AVF}, enlarged. Red arrows highlight states where $\delta$ and $\dpi$ have opposite signs.\label{fig:alignment-and-AVF-big}}
}
\end{figure*}

\begin{figure*}[htb]
\center{
\includegraphics[width=5.7in]{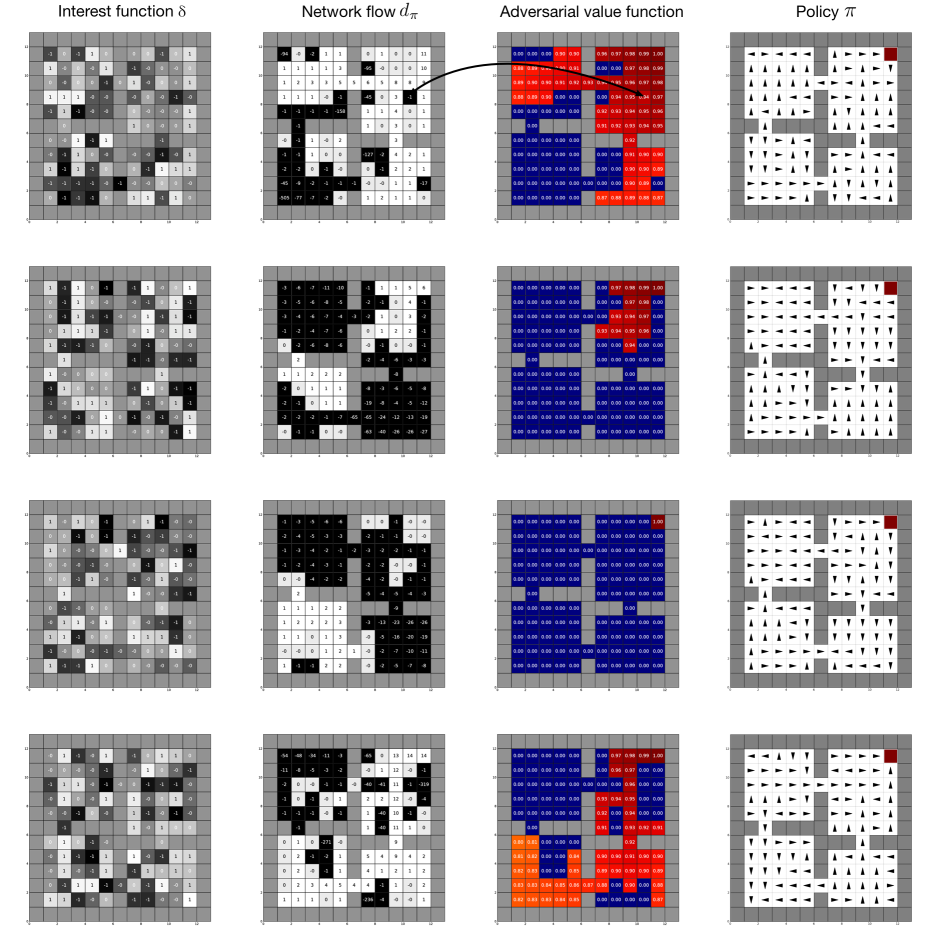}
\caption{Four interest functions sampled from $\{-1, 1\}^n$, along with their corresponding flow $\dpi$, adversarial value function, and corresponding policy. The top example was chosen to illustrate a scenario where $\dpi(x) < 0$ but $V^\pi(x) > 0$; the other three were selected at random. %
In our experiments, sampling from $[-1, 1]^n$ yielded qualitatively similar results.\label{fig:more-interest-functions}}
}
\end{figure*}

\begin{figure*}[htb]
\center{
\includegraphics[width=5.7in]{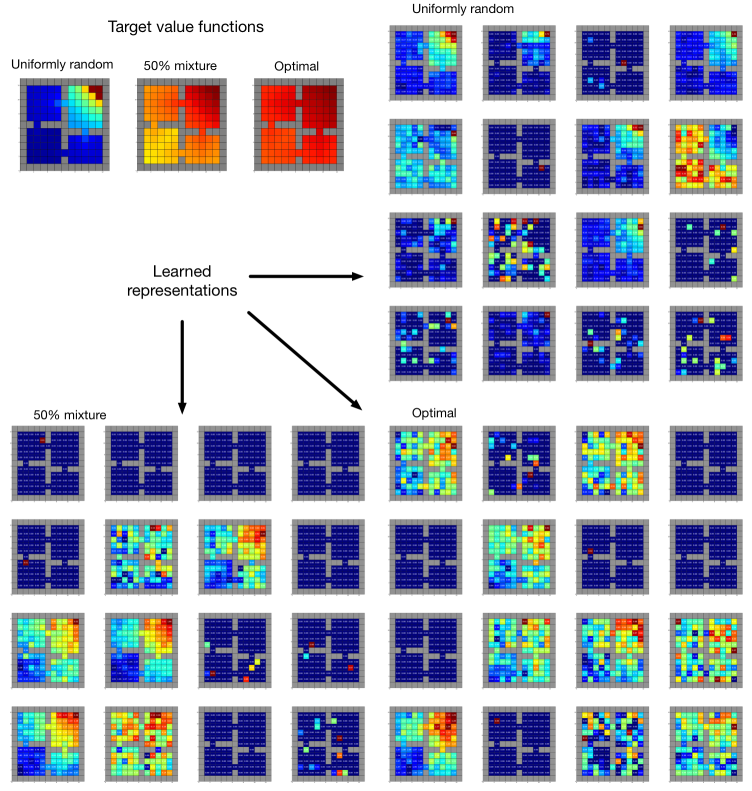}
\caption{16-dimensional representations learned by training a deep network to predict the value function of a single policy, namely: the uniformly random policy, the optimal policy, and a convex combination of the two in equal proportions. \label{fig:learned-value-representations}}
}
\end{figure*}

\begin{figure*}[htb]
\center{
\includegraphics[width=5.7in]{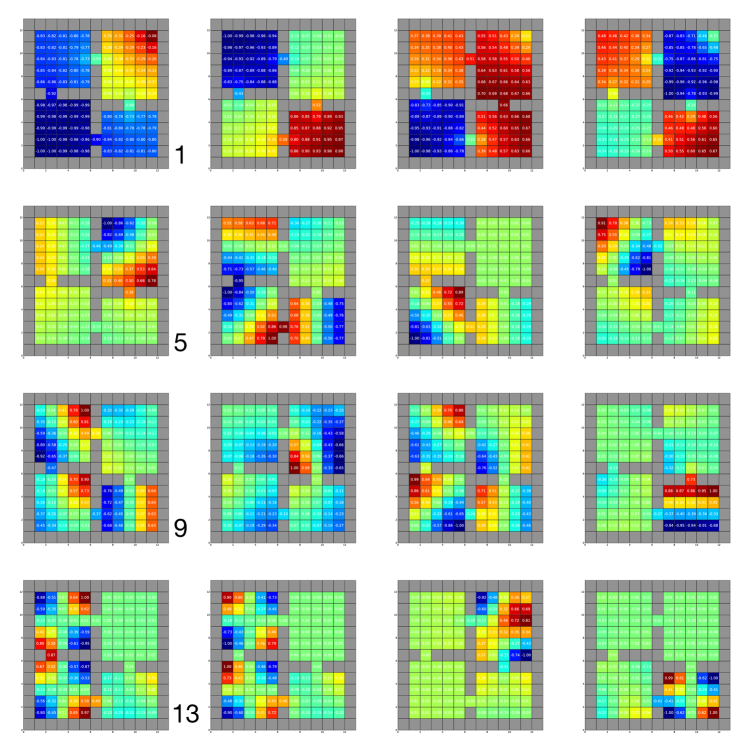}
\caption{16-dimensional representation generated by the proto-value function method \citep{mahadevan07proto} applied to left-singular vectors of the transition function corresponding to the uniformly random policy. The top-left feature, labelled '1', corresponds to the second largest singular value. Notice the asymmetries arising from the absorbing goal state and the walls.\label{fig:proto-value-functions}}
}
\end{figure*}

\begin{figure*}[htb]
\center{
\includegraphics[width=5.7in]{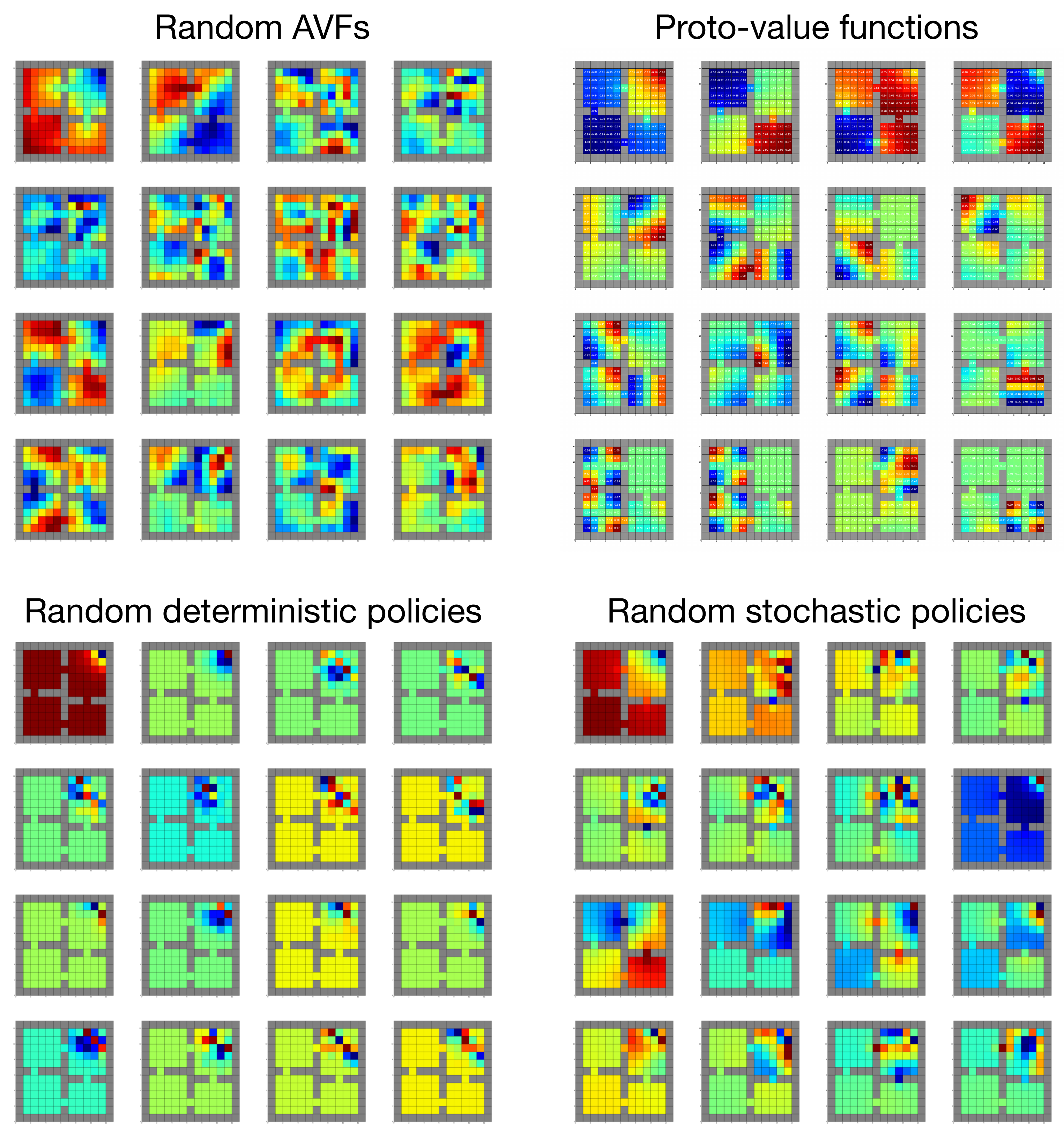}
\caption{16-dimensional representations generated from the principal components of different sets of value functions. Beginning in the top-left corner, in clockwise order: from $k=1000$ AVFs sampled according as in \ref{alg:AVF_representation_learning}; proto-value functions (\ref{fig:proto-value-functions}); from $k=1000$ random deterministic policies (RDPs); and finally from $k=1000$ random stochastic policies. Of the four, only PVFs and AVFs capture the long-range structure of the four-room domain.\label{fig:representations-from-svd}}
}
\end{figure*}

\begin{figure*}[htb]
\center{
\includegraphics[width=5.7in]{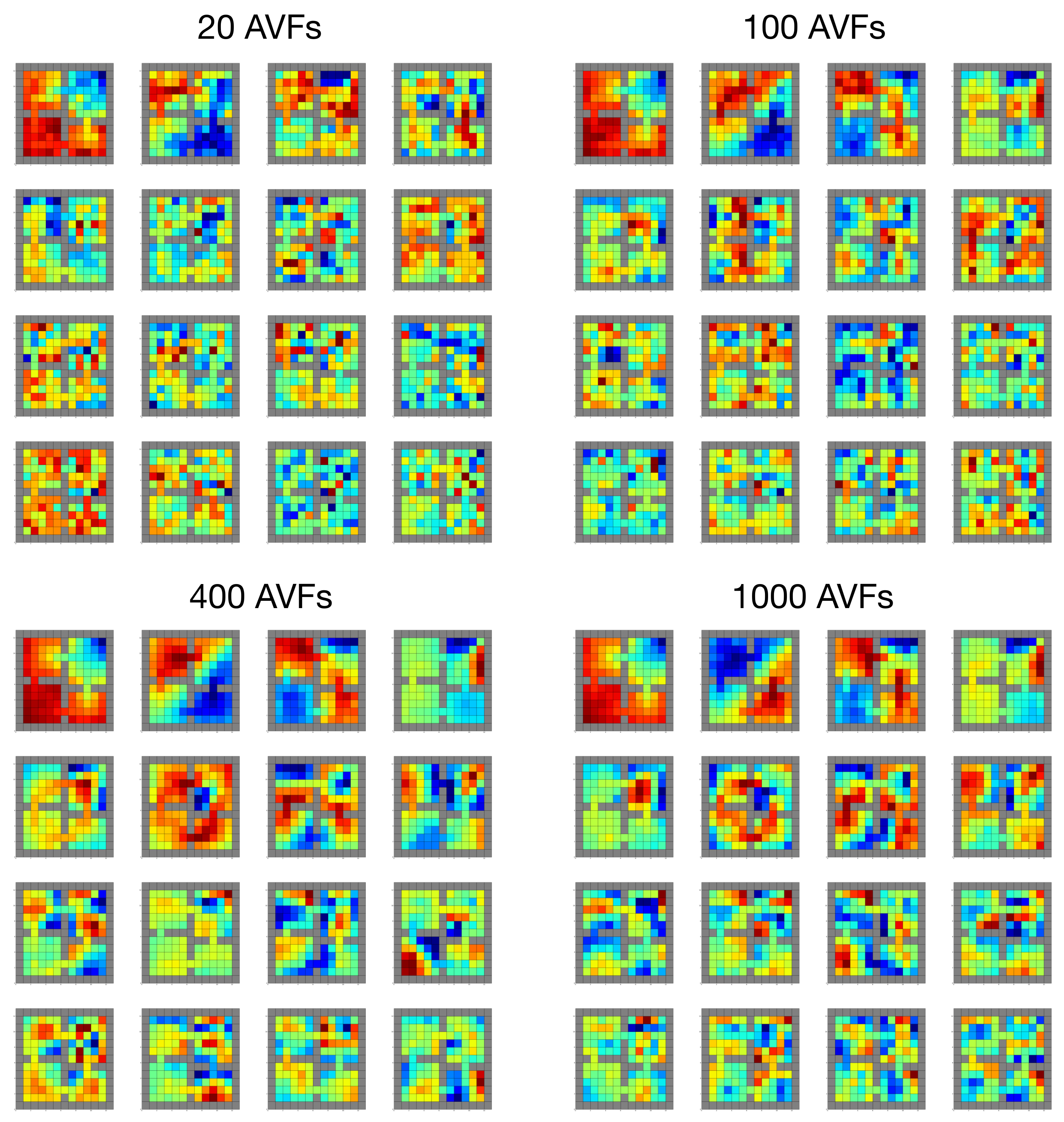}
\caption{16-dimensional representations generated from the principal components of sets of AVFs of varying sizes ($k=20, 100, 400, 1000$). To minimize visualization variance, each set of AVFs contains the previous one. The accompanying video at \url{https://www.youtube.com/watch?v=q_XG7GhImQQ} shows the full progress from $k=16$ to $k=1024$.\label{fig:variable_number_of_avfs}}
}
\end{figure*}

\begin{figure*}[htb]
\center{
\includegraphics[width=5.7in]{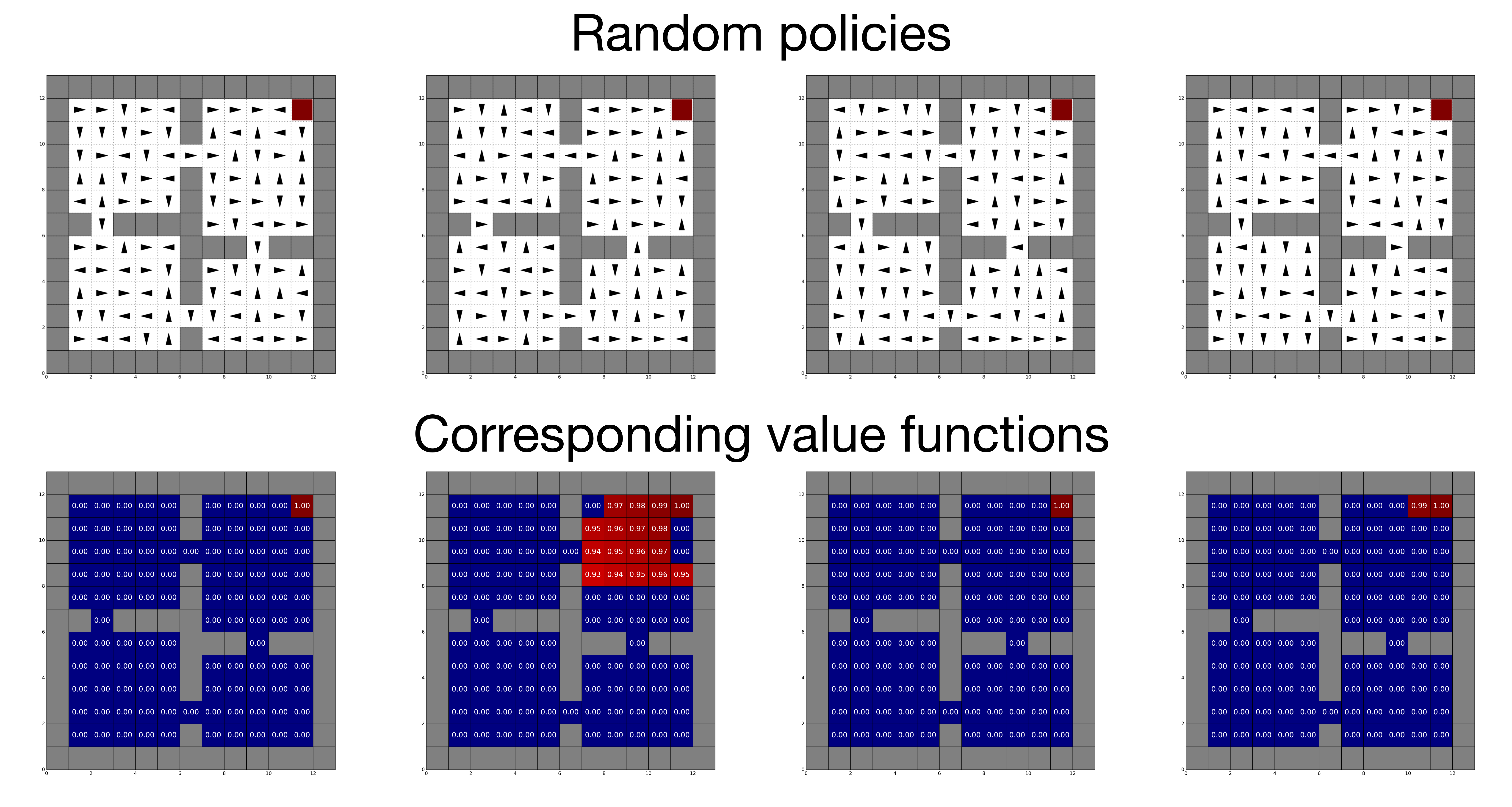}
\caption{A sample of random deterministic policies, together with their corresponding value functions. These policies are generated by assigning a random action to each state. Under this sampling scheme, it is unlikely for a long chain of actions to reach the goal, leading to the corresponding value functions being zero almost everywhere.\label{fig:random-policies-vfs}}
}
\end{figure*}

\begin{figure*}[htb]
\center{
\includegraphics[width=\textwidth]{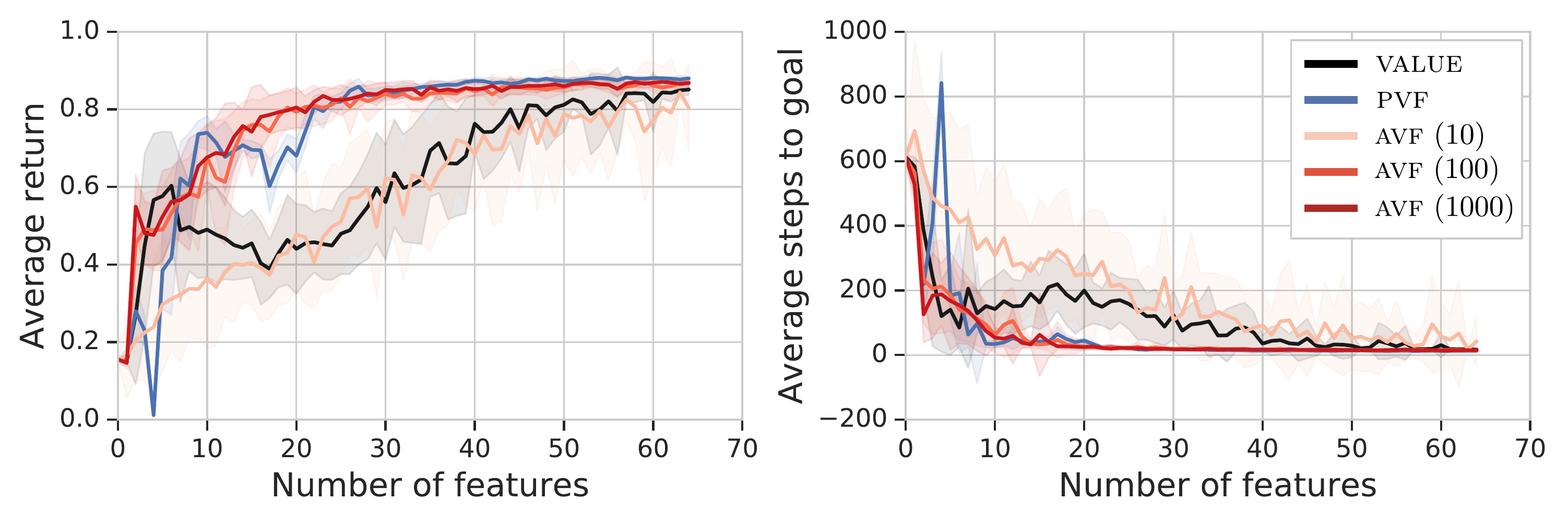}
\caption{Average return (left) and average steps to goal (right), achieved by policies learned using a representation, with given number of features, produced by $\textsc{value}$, $\textsc{avf}$, or $\textsc{pvf}$. Average is over all states and $20$ random seeds, and shading gives standard deviation.\label{fig:sarsa_both}}
}
\end{figure*}

\end{document}